\documentclass{article}

\usepackage[final]{neurips_2025}

\usepackage[utf8]{inputenc} 
\usepackage[T1]{fontenc}    
\usepackage{hyperref}       
\usepackage{url}            
\usepackage{booktabs}       
\usepackage{amsfonts}       
\usepackage{nicefrac}       
\usepackage{microtype}      

\usepackage[dvipsnames]{xcolor} 
\usepackage{pifont}         
\usepackage{graphicx}
\usepackage{import}
\usepackage{caption}
\usepackage{float}
\usepackage{subcaption}

\newfloat{eqfloat}{htbp}{loq}
\floatname{eqfloat}{Eq.}

\title{
Structured Sparse Transition Matrices\\to Enable State Tracking in State-Space Models}

\author{
Aleksandar Terzi\'c$^{1,2}$\thanks{Equal contribution. 
}\\
{\tt\small aleksandar.terzic1@ibm.com}\\
\And
Nicolas Menet$^{1,2}$\footnotemark[1]\\
{\tt\small nicolas.menet@ibm.com}\\
\AND
Michael Hersche$^{1}$\\
{\tt\small michael.hersche@ibm.com}\\
\And
Thomas Hofmann$^{2}$\\
{\tt\small thomas.hofmann@inf.ethz.ch}\\
\And
Abbas Rahimi$^{1}$\\
{\tt\small abr@zurich.ibm.com}
\And
{
\normalfont $^{1}$IBM Research -- Zurich, $^{2}$Department of Computer Science, ETH Zürich}
}

\newif\ifdetailed
\detailedtrue   

\usepackage{mathtools, amsthm, amsfonts, amssymb, dsfont}
\usepackage{algorithm, algpseudocode}
\usepackage{hyperref}
\usepackage{natbib}
\usepackage{bm} 
\usepackage{enumitem}
\usepackage{breakcites}
\usepackage{makecell} 

\usepackage{tabularx}

\setcitestyle{aysep={}} 

\newtheorem{defn}{Definition}
\newtheorem{prop}{Proposition}

\newtheorem{lem}{Lemma}
\newtheorem{example}{Example}

\newtheorem{theorem}{Theorem}
\newtheorem{property}{Property}

\newtheorem{manualpropinner}{Proposition}
\newenvironment{manualprop}[1]{%
  \IfBlankTF{#1}
    {\renewcommand{\themanualpropinner}{\unskip}}
    {\renewcommand\themanualpropinner{#1}}%
  \manualpropinner
}{\endmanualpropinner}

\newtheorem{manualproperinner}{Property}
\newenvironment{manualproper}[1]{%
  \IfBlankTF{#1}
    {\renewcommand{\themanualproperinner}{\unskip}}
    {\renewcommand\themanualproperinner{#1}}%
  \manualproperinner
}{\endmanualproperinner}

\usepackage{adjustbox, graphicx, import, xparse, pgfplots}
\usepgfplotslibrary{external}
\everymath=\expandafter{\the\everymath\displaystyle}
\NewDocumentCommand{\incplt}{O{\columnwidth}m}{%
  \begin{center}
    \adjustbox{width=#1}{\import{./graphics/}{#2.pgf}}
 \end{center}
}
\pgfplotsset{compat=1.18}

\usepackage{booktabs}
\usepackage{siunitx}
\usepackage{multirow}
\usepackage[normalem]{ulem}

\definecolor{first}{RGB}{34,139,34}
\definecolor{second}{RGB}{0,0,255}   
\definecolor{third}{RGB}{148,0,211}  

\usepackage{array}
\usepackage{wrapfig}
\usepackage{xcolor}

\usepackage{titlesec}
\titlespacing*{\section}{0pt}{1ex plus .5ex minus .2ex}{0.8ex}
\titlespacing*{\subsection}{0pt}{0.8ex plus .3ex minus .2ex}{0.5ex}

\usepackage{xspace} 
\newcommand{\name}{PD-SSM\xspace}

\usepackage{xcolor}
\pagecolor{white}      
\color{black}          

\begin{document}

\maketitle

\begin{abstract}

%
%
Modern state-space models (SSMs) often utilize structured transition matrices which enable efficient computation but pose restrictions on the model's expressivity, as measured in terms of the ability to emulate finite-state automata (FSA).
While unstructured transition matrices are optimal in terms of expressivity, they come at a prohibitively high compute and memory cost, even for moderate state sizes.
We propose a structured sparse parametrization of transition matrices in SSMs that enables FSA state tracking with provably optimal state size and depth, while keeping the computational cost of the recurrence comparable to that of diagonal SSMs.
Our method, \emph{PD-SSM}, parametrizes the transition matrix as the product of a column one-hot matrix ($P$) and a complex-valued diagonal matrix ($D$). 
As a result, the computational cost of parallel scans scales linearly with the state size.
%
Theoretically, the model is BIBO-stable and can emulate any $N$-state FSA with one layer of dimension $N$ and a linear readout of size $N \times N$, significantly improving on all current structured SSM guarantees.
Experimentally, the model significantly outperforms a wide collection of modern SSM variants on various FSA state tracking tasks.
On multivariate time-series classification, it outperforms neural controlled differential equations, a paradigm explicitly built for time-series analysis.
Finally, we integrate PD-SSM into a hybrid Transformer-SSM architecture and demonstrate that the model can effectively track the states of a complex FSA in which transitions are encoded into sets of variable-length English sentences.
The code is available at \href{https://github.com/IBM/expressive-sparse-state-space-model}{https://github.com/IBM/expressive-sparse-state-space-model}.
%
%

\end{abstract}

\section{Introduction}

The Transformer \citep{vaswani_attention_2017} marked a paradigm shift in machine learning, providing a unified yet versatile architecture with strong scalability. 
%
However, experimental and theoretical results indicate that Transformers struggle in algorithmic state tracking tasks such as finite-state automaton (FSA) emulation~\citep{hahn_2020_theoretical, bhattamishra_ability_2020, merrill_parallelism_2023, deletang_neural_2023, liu_transformers_2023, strobl_2024_survey}. 
Moreover, the inherent quadratic cost of the attention mechanism hinders the training on very long sequences, and its linearly growing key/value cache can lead to memory issues during inference.
%
%
%

In contrast, \emph{state-space models} (SSMs)~\citep{gu_efficiently_2022, gupta_diagonal_2022, fu_hungry_2023, smith_simplified_2023, orvieto_resurrecting_2023, gu_mamba_2023} offer a scalable alternative to Transformers with complementary properties: they generalize better to longer sequences, the number of total operations scales linearly with the sequence length, and memory is constant during inference.
%
%
%
Modern large language models (LLMs) are increasingly adopting a hybrid architecture that integrates both Transformer and SSM layers~\citep{yang2024parallelizing, ren_2024_samba, de_griffin_2024, waleffe_2024_empirical, lenz_2025_jamba, wu2025transxssm}. 
%
%
%
To achieve computational efficiency competitive with that of Transformers, modern SSMs often utilize diagonal transition matrices. Empirically, in time-invariant SSMs, this strategy has proven effective in a range of applications~\citep{gupta_diagonal_2022, gu_parameterization_2022, orvieto_resurrecting_2023}, while time-varying diagonal SSMs can achieve language modeling performance comparable to that of the Transformer~\citep{gu_mamba_2023}.
%
%
%

Despite their promising performance in language modeling, diagonal SSMs are restricted in the type of FSA they can emulate~\citep{merrill_illusion_2024}. 
This stands in stark contrast with the nonlinear RNN, which can quickly learn such algorithmic abstractions~\citep{deletang_neural_2023}.
%
%
To enable FSA \emph{state tracking}, previous approaches propose using transition matrices of various structures such as fully unstructured (i.e., dense) ~\citep{terzic2025sdssm}, or (semi-)structured matrices such as block-diagonal~\citep{fan_advancing_2024}, diagonal plus low-rank (DPLR)~\citep{schlag2021linear, yang2024parallelizing, walker_2025_structured, peng_2025_rwkv} or products of DPLR matrices~\cite{siems_2025_deltaproduct}.
While any $N$-state FSA can be represented by an unstructured SSM with a single layer, state dimensionality $N$, and a readout size $N \times N$, unstructured SSMs scale unfavorably and thus prohibit large-scale training. Current (semi-)structured approaches are more efficient, but do not allow for such compact encodings.
%
%
%
%
%
With this in mind, the goal of this work can be summarized as follows:

\begin{center}
\textit{We aim to enable single-layer time-varying SSMs with state size $N$ and readout size $N \times N$ to represent any FSA with $N$ states while keeping the cost comparable to that of diagonal SSMs.}
\end{center}

To this end, we propose a novel structured sparse matrix parametrization of SSM transition matrices, denoted as \emph{PD} parametrization, defined as a product of a binary column-one-hot matrix ($P$) and a complex-valued diagonal matrix ($D$).
A \emph{PD} parametrization of transition matrices is preserved under multiplication, meaning that the product of \emph{PD} matrices is still a \emph{PD} matrix. As such, a chained product of $L$ \emph{PD} matrices can be computed using parallel scans~\citep{blelloch_prex_1990, martin_parallelizing_2018, smith_simplified_2023} in $\Theta(L N)$ concurrent operations and $\Theta(L N)$ memory, matching the complexity and concurrency of a chain of diagonal matrix products.
%
%
In terms of expressivity, the \emph{PD} matrices encompass the commonly used diagonal transition matrices, but can also represent transition functions of arbitrary $N$-state automata. 
Our contributions are as follows:
\begin{itemize}
    \item We propose \name, a time-varying SSM that can emulate any $N$-state FSA with a single layer, state dimension $N$, and a linear readout of size $N \times N$. Theoretically, we prove that \name{} is stable and that it achieves universal FSA emulation with (almost) the least possible state size.
    
    \item 
    Empirically, we demonstrate that \name learns to track the states of complex automata by exhibiting state-of-the-art length generalization. We furthermore test our model on long-range multivariate time-series classification tasks and observe that it reaches state-of-the-art performance, even surpassing neural controlled differential equations. 
    
    \item We provide a novel benchmark 
    in which each FSA transition is redundantly encoded by meaningful English sentences. Adopting a hybrid Transformer-SSM architecture, we show that integrating \name enables the model to emulate the underlying FSA and to generalize to longer sequences, in contrast to diagonal SSMs which fail on in-domain sequences. \looseness -1
\end{itemize}

\section{Background}

\subsection{State-Space Models}\label{sec:background_ssms}

By \emph{state-space models} (SSMs), we denote neural networks that utilize the recurrence equations 
\begin{align} \label{eq_slssm}
x_{t} &= A(u_t) x_{t-1} + B (u_t) u_t \nonumber\\
y_{t} &= C(u_t)x_{t} + D(u_t)u_t \\\nonumber
o_{t} &= \psi(y_t)\nonumber
\end{align}
with $u_t \in \mathbb R^D$ the input to the system, $x_t \in \mathbb C^N$ its (hidden) state, $y_t \in \mathbb C^D$ the complex-valued output, and $o_t \in \mathbb R^D$ the real-valued output. $A(u_t) \in \mathbb{C}^{N \times N}$ is the state transition matrix, $B(u_t) \in \mathbb{C}^{D\times N}$ and $C(u_t) \in \mathbb{C}^{D\times N}$ map from embeddings to states and vice versa, and $D(u_t) \in \mathbb{C}^{D\times D}$ is a memoryless skip connection. Finally, $\psi:\mathbb C^D \to \mathbb R^D$ extracts a real-valued embedding from a complex-valued embedding, often implemented as $\psi(y_t)=Re\{y_t\}$~\citep{gu_parameterization_2022, gupta_diagonal_2022, orvieto_resurrecting_2023}.
The time-variance of the system stems from the time dependence of $A(u_t), B(u_t), C(u_t), D(u_t)$. In this work we only consider time dependence in the $A(u_t)$ matrix, this being the crucial differentiator between time-varying and time-invariant SSMs~\citep{gu_mamba_2023, merrill_illusion_2024}.  We thus fix $B(u_t)=B, C(u_t)=C, D(u_t)=D$.
\setcounter{footnote}{0}

If the system is time-invariant, i.e., $A(u_t)=A$, then its dynamics can be equivalently represented by a diagonal transition matrix up to an arbitrarily small perturbation of the entries of $A$~\citep{orvieto_resurrecting_2023, axler_linear_2024}.\footnote{$\forall \varepsilon > 0 , A \in \mathbb{R}^{N\times N} \ \exists A_\varepsilon \in \mathbb{R}^{N \times N} \text{ s.t. } A_\varepsilon=W\Lambda W^{-1} \text{ with diagonal }\Lambda \in \mathbb{C}^{N \times N} \text{ and } \|A - A_\varepsilon\|_F < \varepsilon.$}
%
Indeed, writing $A=W\Lambda W^{-1}$ results, upon change of basis $\tilde{x}_t:=W^{-1}x_t$, in an equivalent state evolution $\tilde{x}_t=\Lambda\tilde{x}_t+W^{-1}Bu_t$ but with a diagonal state transition matrix $\Lambda$. 
%

In contrast, if $A(u_t)$ is a function of time, then the system admits a diagonal representation if and only if all $A(u_t)$ are diagonalizable \emph{under the same change of basis}. That is, there must exist a single $W\in\mathbb{C}^{N\times N}$ such that for all $t$, $A(u_t)=W\Lambda_tW^{-1}$. This is a much more restrictive condition. It can in fact only be fulfilled if the matrix product commutes for all state transition matrices occurring over time, i.e., $\forall i\neq j,  A(u_i)A(u_j)=A(u_j)A(u_i)$~\citep{axler_linear_2024}.
%
This fact hints at the restrictiveness of diagonal transition matrices, which is expanded upon in an alternative framework in Section~\ref{sec:background_limitations}

\begin{table}[t]
\centering
\resizebox{\textwidth}{!}{
\begin{tabular}{rcccc}
\toprule
Matrix Structure          & Example Models  &   Solvable                &            Non-Solvable           & Cost \\
\cmidrule(r){1-1} \cmidrule(r){2-4} \cmidrule(r){4-5}

$\mathbb{R}$ Diagonal &   Mamba, S6, S7, GLA, mLSTM   & \textcolor{Black}{\ding{55}}               &            \textcolor{Black}{\ding{55}}   &    \textcolor{Black}{$\Theta(LN)$}       \\ \cmidrule(r){1-1} \cmidrule(r){2-4} \cmidrule(r){4-5}

$\mathbb{C}$ Diagonal &  PD-SSM with $P=\mathbb{I}_N$ (Ours)  &  \textcolor{Black}{\checkmark}*                &    \textcolor{Black}{\ding{55}}       & \textcolor{Black}{$\Theta(LN)$}             \\ \cmidrule(r){1-1} \cmidrule(r){2-4} \cmidrule(r){4-5}

DPLR &   (Gated) DeltaNet, DeltaProduct, RWKV-7    &  \textcolor{Black}{\checkmark}*                &    \textcolor{Black}{\checkmark}*       & \textcolor{Black}{$\Theta(L^2N)$}             \\ \cmidrule(r){1-1} \cmidrule(r){2-4} \cmidrule(r){4-5}

$\mathbb{R}$ Dense    &   SD-SSM &     \textcolor{Black}{\checkmark}             &          \textcolor{Black}{\checkmark}                  &   \textcolor{Black}{$\Theta(LN^3)$}       \\ \cmidrule(r){1-1} \cmidrule(r){2-4} \cmidrule(r){4-5}

\textbf{Structured Sparse}     &  \textbf{PD-SSM (Ours)} &   \textcolor{Black}{\checkmark}               &      \textcolor{Black}{\checkmark}   & \textcolor{Black}{$\Theta(LN)$}            \\ 
\bottomrule    
\end{tabular}
}
\vspace{0.5em}
\caption{The structure of the transition matrix $A(u_t)$ in time-varying SSMs determines the class of FSA they can emulate. We consider two classes of automata, partitioned based on the solvability of their transformation group. \textcolor{Black}{\checkmark} indicates that the structure enables the emulation of any $N$-state FSA from the partition using one layer, state dimension $N$ and a linear readout of dimension $N\times N$. \textcolor{Black}{\checkmark}* indicates that the matrix structure can emulate any automaton from the partition, but at potentially high dimension or depth, and in case of RWKV-7 exponentially large linear layers~\cite{siems_2025_deltaproduct}. The cost only considers the complexity of the parallel computation of~$x_t$. Modern DPLR methods trade off FLOPs for increased parallelism and better memory management. The recurrent complexity is $\Theta(LN^2)$, but modern methods utilize algorithms that scale quadratically in $L$~\citep{yang2024parallelizing}.
}
\label{tab:ssm_fsa_tracking_expressivity}
\end{table}

\subsection{Modeling Finite-State Automata with SSMs}\label{sec:background_mapping}

In this work, we measure the expressivity of an SSM by considering the type of deterministic finite-state automaton (FSA) it can emulate.
A deterministic FSA is an abstract model of computation defined as a 5-tuple $(Q, \Sigma, \delta, q_{\text{init}}, F)$ where $Q$ is a finite set of states, $\Sigma$ is a finite input alphabet, $\delta : Q \times \Sigma \rightarrow Q$ is the state transition function, $q_{\text{init}} \in Q$ is a fixed initial state, and $F \subseteq Q$ is the set of accepting states. 
%
%
%

Any (deterministic) FSA can be mapped to a time-variant SSM as follows.
%
%
Encode each $q \in Q$ using $enc: Q \rightarrow\mathbb{R}^{|Q|}$ such that the encodings of different states are orthogonal. Note that orthogonality is a sufficient but not necessary condition for mapping an FSA to a selective SSM, as certain automata such as modular counters allow alternative, more compact, mappings to SSMs.
Given such an orthogonal encoding of states, we can map the state transition function $\delta : Q \times \Sigma \rightarrow Q$ to the state transition matrices $A(u_t)$ 
%
%
via $A(\sigma)={\scriptstyle\sum\nolimits}_{q \in Q}\ enc(\delta(q, \sigma))\cdot enc(q)^T$,
%
%
set $B =0$, $C=\mathbb{I}$, $D = 0$, and $\psi = id$. Upon identification of $q$ with $enc(q)$ and $q_{\text{init}}$ with $x_0$, the SSM matches the FSA exactly.

\subsection{Limitations of Time-Variant SSMs for FSA State Tracking}\label{sec:background_limitations}

%
To demonstrate the limitations of various SSMs structures for state tracking in automata, it suffices to consider automata with group structure in their state transitions. The transformation group~\citep{straubing_book} of such an automaton is then the algebraic group of invertible state-to-state mappings with function composition as a binary operation\footnote{A more complete background on algebra and results on SSM expressivity is provided in Appendix~\ref{app:background}.}.
%
%
For conciseness, we equate such automata with their transformation group, thus saying, for instance, \emph{solvable automaton.}

Through a circuit complexity argument, \cite{merrill_illusion_2024} showed that non-solvable automata cannot be emulated using bounded-depth logarithmic precision diagonal time-variant SSMs, providing an upper bound on their expressivity. Through an explicit construction~\cite{sarrof2024expressivecapacitystatespace} demonstrate that all solvable automata can be emulated using finite-precision complex-valued diagonal time-variant SSMs, although the depth of the SSM stack exhibits a non-trivial dependence on the complexity of the FSA. 
Concretely, the SSM stack depth is proportional to the Krohn-Rhodes complexity of the transformation group~\cite{margolis_2024_decidability}.
A more recent family of time-variant SSM models utilizes diagonal plus low-rank (DPLR) transition matrices~\cite{grazzi2024unlocking,peng_2025_rwkv, siems_2025_deltaproduct}. While such matrices enable the emulation of any FSA, the necessary model depth is either a function of the automaton's structure, or the model requires exponentially large linear layers~\cite{siems_2025_deltaproduct}.
In contrast, unstructured SSMs~\cite{terzic2025sdssm, walker_2025_structured} enable the emulation of any $N$-state FSA using a single layer, a state size of $N$ and a readout size $N \times N$. However, they scale poorly with the state size. 
For an overview of the results, see Table~\ref{tab:ssm_fsa_tracking_expressivity}.

\subsection{Associative Scan for Fully Parallelizable Recurrence in SSMs}\label{sec:background_pscan}

We use parallel scans to concurrently compute the states of an SSM~\citep{blelloch_prex_1990, martin_parallelizing_2018,  smith_simplified_2023}.
Suppose a binary associative operator $\bullet$, i.e., $(a \bullet b) \bullet c = a \bullet (b \bullet c)$,
and a sequence of $L$ elements $[a_1, a_2, \ldots, a_L]$.
The scan operation (sometimes referred to as \emph{all-prefix-sum}) then returns the sequence
$
[a_1, (a_1 \bullet a_2), \ldots, (a_1 \bullet a_2 \bullet \cdots \bullet a_L)].
$
Note that composition of the linear recurrence of a (time-variant) SSM $x_t = A_t x_{t-1} + b_t$ takes precisely the form of a parallel scan $(A_{t+1}, b_{t+1}) \bullet (A_t, b_t) \mapsto (A_{t+1} A_t, A_{t+1} b_t + b_{t+1})$ with prefix elements $a_t = (A_t, b_t)$. According to~\citeauthor{blelloch_prex_1990} (\citeyear{blelloch_prex_1990}, Section~1.4), associative scans can be computed in $\Theta(T_\odot\, L)$ total operations (i.e., in work-efficient manner) with $\Theta(\log L)$ sequential steps, where $T_\odot$ represents the cost of matrix-matrix multiplication. For dense matrices $A_t \in \mathbb{R}^{N \times N}$,
$T_\odot = O(N^3)$ and thus the associative scan quickly becomes prohibitively expensive
in deep learning settings. In contrast, if $A_t$ are diagonal matrices, $T_\odot = O(N)$. In terms of memory, parallel scans with dense matrices use $\Theta(N^2 L)$ memory, whereas diagonal matrices lead to a $\Theta(N L)$ memory consumption.
\section{PD Parametrization for Efficient and Expressive Transition Matrices}

%

Prior research indicates that the expressiveness of time-varying SSMs can be significantly enhanced by relaxing the structural constraints of the transition matrix $A(u_t)$ to allow for non-diagonal structures. 
\citep{merrill_illusion_2024, cirone_theoretical_2024,terzic2025sdssm}. 
As noted by \cite{terzic2025sdssm}, a relaxation to arbitrary matrices prohibitively increases the compute cost of parallel scans from $\Theta(LN)$ to $\Theta(LN^3)$.
To circumvent these limitations, we propose to parameterize the transition matrices as the product $A(u_t)=P(u_t)D(u_t)$, where $P(u_t)$ is a binary matrix in which each column has a single non-zero element, and $D(u_t)$ is a complex-valued diagonal matrix. 

Figure~\ref{fig:model-sketch} presents our architecture for generating transition matrices with PD parametrization.
The diagonal matrices are generated by two nonlinear feed-forward neural networks, each with a single hidden layer employing the GeLU nonlinearity $\sigma_{\text{gelu}}$ and a saturating nonlinearity in the form of the sigmoid $\sigma$. The two networks define a magnitude generator $|D(u_t)|$ and a phase generator $\phi(D(u_t))$:
\begin{align*}
    |D(u_t)| &= \sigma(W^{M}_o (\sigma_{\text{gelu}}(W_{i}^{M}u_t + b_i^M) +b_o^M)) \in (0,1)^N\\
    \phi(D(u_t)) &= 2 \pi \sigma(W^{P}_o (\sigma_{\text{gelu}}(W_{i}^{P}u_t + b_i^P) + b_o^P)) \in (0,2\pi)^N
\end{align*}
%

Concerning the column one-hot matrix $P$, the input $u_t$ generates the weights $s(u_t)$, which are used to soft-select among a set of trainable transition matrices, written as the dictionary $\{M_i \in \mathbb{R}^{N\times N}\}_{i\in[K]}$ with $K$ being a hyperparameter. 
Sparsity of $P$ is achieved by applying a column-wise hardmax after soft-selection of the state transition. More formally:
\begin{align*} 
        s(u_t)&=\text{softmax}(Su_t)\in\Delta^{K-1}  \\
        M(u_t)&=\sum_{k=1}^K s_k(u_t)M_k\in\mathbb{R}^{N\times N} \\
        P_{:,j}(u_t)&=\text{hardmax}(M_{:,j}(u_t))\in\{0,1\}^N \text{ where } \text{hardmax(x)}_{i} := \delta_{i, \arg\max_j x_j}
\end{align*}

The parametrization is motivated by~\cite{terzic2025sdssm}, which has shown that models utilizing normalized~\cite{fan_advancing_2024} convex combinations of dense transition matrices can achieve perfect length generalization on FSA emulation tasks and are parameter-efficient compared to alternative proposals such as e.g.~\citep{hasani2023liquid, fan_advancing_2024, merrill_illusion_2024}.

\begin{figure}[t]
    \centering
    \includegraphics[width=0.95\linewidth]{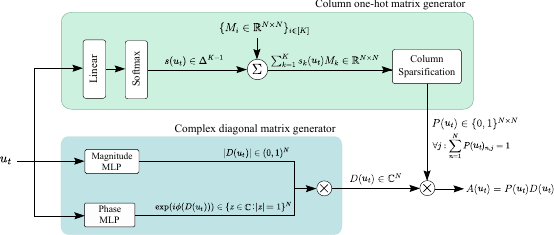}
    \caption{The PD parametrization can be integrated into any selective SSM by adopting the shown architecture for generation of structured sparse state transition matrices $A(u_t) = P(u_t) D(u_t)$.} 
    \label{fig:model-sketch}
\end{figure}
The $P$ and $D$ factors of our parametrization have complementary strengths for encoding automata into time-varying SSMs. The $P$ matrix enables emulating any FSA, but the required dimension scales linearly with the number of states. For cyclic automata, which form a central building block of all solvable automata~\citep{krohn_algebraic_1965,sarrof2024expressivecapacitystatespace}, complex diagonal matrices provide a more compact encoding compared to column one-hot matrices. Visual intuition is provided in Figure~\ref{fig:automaton_transition_figure}. 
Additionally, as proven in the following section, the diagonal matrices provide a guarantee for the system's BIBO stability as the magnitude of each entry lies in $(0,1)$.

\begin{figure}[h]
    \centering
    \begin{subfigure}[b]{0.45\textwidth}
        \centering
        \includegraphics[width=\textwidth]{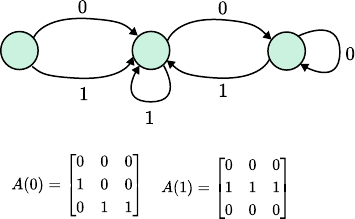}
        \caption{A non-cyclic FSA and its two corresponding column one-hot transition matrices.}
    \end{subfigure}
    \hfill
    \begin{subfigure}[b]{0.4\textwidth}
        \centering
        \includegraphics[width=\textwidth]{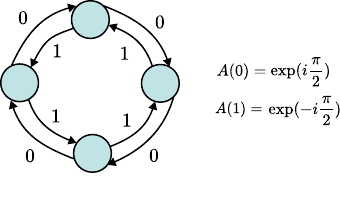}
        \caption{A cyclic FSA whose behavior can be emulated with a single complex number.}
    \end{subfigure}
    \caption{Any $N$-state FSA can be encoded using sparse binary $N\times N$ transition matrices 
    (a), but modular counters admit a more compact representation based on diagonal transition matrices 
    (b).}
    \label{fig:automaton_transition_figure}
\end{figure}



\subsection{Surrogate Gradients}

Strictly speaking, $\tfrac{\partial P(u)}{\partial M(u)}$ is a generalized function—it vanishes almost everywhere, except at isolated points where it exhibits Dirac delta-like singularities. To smooth these singularities over sets of non-zero measure, we approximate the hardmax with softmax during the backward pass.
\begin{equation*}
    \frac{\partial P}{\partial M} = \frac{\partial \text{hardmax} (M)}{\partial M} \approx \frac{\partial \text{softmax}(M)}{\partial M} \\
\end{equation*}
This is reminiscent of the \emph{slope-annealed straight-through estimator}~\citep{bengio_stochastic_2013, paulus_rao_2021}, yet we use stochasticity neither in the forward nor in the backward pass (ablations with stochastic categorical sampling in $P$ are reported in Appendix~\ref{app:results}).
Note that during the forward pass we do not relax the hardmax, since doing so would break the sparsity essential for an efficient parallel scan.
With tempered softmax in the low-temperature limit, the above expression becomes exact.

\subsection{Algebraic Structure of the PD Parametrization}

In this section, we formalize the set of transition matrices $\mathbb H^{N \times N}$ that are decomposable into a binary one-hot column matrix $P$ and a complex diagonal matrix $D$. Note that our $PD$ matrices exhibit such a column one-hot structure. Hence, all of the subsequent statements made for the set $\mathbb{H}$ immediately also hold for our matrix parametrization. All proofs are in Appendix~\ref{app:proofs}.

\begin{defn}[One-hot Column Matrices]
    Let $\mathbb H^{M \times N} := \{A \in \mathbb C^{M \times N} :  \forall i \ \|A_{:, i}\|_0 = 1\}$ where $\|x \|_0$ denotes the $\ell^0$-"norm" counting the number of non-zero entries in $x$.
\end{defn}

Under matrix multiplication as its binary operation, $\mathbb H^{N \times N}$ forms a monoid.
\begin{property}[Algebraic Structure of One-hot Column Matrices]\label{property:algebraic_structure}
    $\mathbb H^{N \times N}$ is a monoid under matrix multiplication, i.e., it is closed under associative matrix multiplication and contains the identity.
\end{property}
Closure under multiplication is essential for efficient chained matrix multiplication via parallel scans, because matrix multiplication in $\mathbb H^{N \times N}$ can be implemented in $\Theta(N)$ instead of the usual $\Theta(N^3)$, as formalized by Property~\ref{property:computational_efficiency} and visualized in Figure~\ref{fig:PD_matmul}.
\begin{property}[Computational Efficiency of Matrix Multiplication in $\mathbb H^{N \times N}$]\label{property:computational_efficiency}
    Let $A,B\in\mathbb{H}^{N\times N}$. Then $C=AB \in \mathbb{H}^{N\times N}$ can be computed in $\Theta(N)$ arithmetic operations.
\end{property}
%
\begin{figure}[h]
    \centering
    \includegraphics[width=0.95\linewidth]{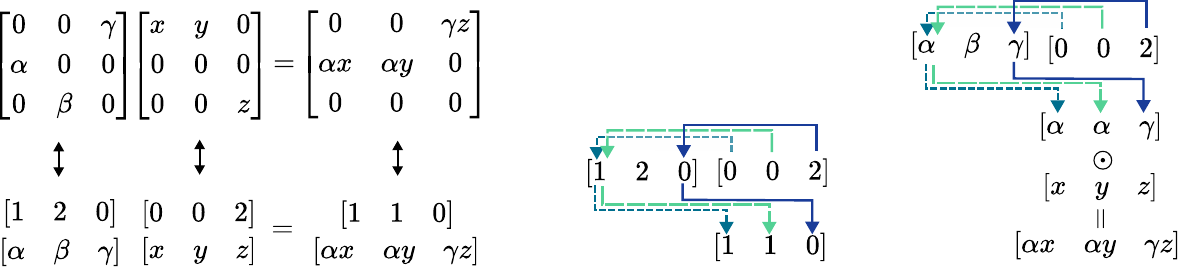}
    \caption{
    %
    \emph{Left:} The sparse matrices in $\mathbb{H}^{N \times N}$ can be efficiently represented by separately storing the indices of the active elements and their values. \emph{Center:} The indices of the matrix product are computed with a gather-scatter operation. \emph{Right:} The nonzero entries of the matrix product are computed using gather-scatter followed by element-wise multiplication.}
    \label{fig:PD_matmul}
\end{figure}

\subsection{Stability and Expressivity of the PD Parametrization}

If the PD parametrization is adopted for state transition matrices and if one ensures matrix entries inside the complex unit circle, then the state space model becomes provably bounded-input, bounded-output (BIBO) stable. Note that our parametrization of $|D(u_t)|$ with a sigmoid-based MLP ensures all conditions for BIBO stability are met.


\begin{prop}[System Stability under PD-Parametrization] Let $\varepsilon \in (0,1]$ and consider the state transition $x_t = A_t x_{t-1}+b_t$ with $A_t \in \mathbb H^{N \times N} : \|A_t\|_\infty \leq 1- \varepsilon$. Let further $\|x_{0}\|_2 \leq B$ and $\|b_t\|_2 \leq B$ for $B \in \mathbb R_+$. Then it holds that
    \begin{equation}
        \|x_t\|_2 \leq \sqrt{N}B / \varepsilon \quad \forall t.
    \end{equation}
\end{prop}


As a direct consequence of the FSA to SSM mapping described in Section~\ref{sec:background_mapping}, Proposition~\ref{prop:expressivity} holds: 

\begin{prop}[Expressivity of PD Parametrization]\label{prop:expressivity}
Any FSA with $N$ states can be exactly represented by a single-layer \name{} with a state size $N$ and linear readout of size $N \times N$.
\end{prop}


Not only can \name{}s represent any FSA, they do so with the (almost) smallest state size possible, i.e., \name{}s are maximally expressive for regular languages.

\begin{prop}[Optimality of PD Parametrization]\label{prop:optimality_pd_parametrization}
    For any $N$ there exists a finite-state automaton with $N$ states that cannot be emulated by any single-layer SSM with state size less than $N-1$ under unique state encodings.
\end{prop}

The assumption of unique state encodings, meaning that each state is represented by a single, unique vector, is a practical necessity for readout. Indeed, without unique state encodings (and with arbitrary precision), even a single-layer real SSM with state size $1$ can represent any FSA, albeit in a format that requires exponentially large lookup tables to read out. Proposition~\ref{prop:readout} makes this statement exact. 

\begin{prop}[Arbitrary Precision and Readout]\label{prop:readout}
    Consider $x_{t+1} = x_t + b_t$ where $b_t = u_t \cdot k_t$ with $u_t \in \mathbb Q$ input encodings and $k_t = \sqrt{p_t} \in \mathbb R /\mathbb Q$ time encodings, where $p_t$ is the $t$-th prime. Then, any FSA can be encoded into this scalar-valued SSM under an appropriate lookup table as readout.
\end{prop}

\section{Results}


\subsection{Runtime Measurements}\label{sec:experiments_runtime}

We first measure how the runtime of a single-layer SSM scales as a function of the transition matrix structure as well as the hidden dimension on an NVIDIA A100-80GB GPU. We compare PD-SSM 
\noindent
\begin{minipage}{0.5\linewidth}
     with the unstructured (dense) real-valued SD-SSM~\citep{terzic2025sdssm}, as well as a diagonal variant of PD-SSM that sets $A(u_t)=D(u_t)\in\mathbb{C}^{N_{c}}$.
    To equalize the effective state size, the state dimensionality $N_r$ of the real-valued dense SSM is twice that of the complex-valued models $N_c$.
    The embedding size is scaled as $D=N_{r}=2N_c$. 
    %
    %
    In Figure~\ref{fig:runtime_dimension} we see that PD-SSM scales significantly better than the dense SSM, with a $71 \times$ speed-up at $D=5632$.
    At this dimension, the diagonal SSM is $7\times$ faster than PD-SSM.
    The PD model's higher runtime mainly stems from the additional operations in the generation of P matrices. For more details and results under different settings, see Appendix~\ref{app:results}. \looseness -1
    
\end{minipage}%
\hfill
\begin{minipage}{0.48\linewidth}
    \centering
    \includegraphics[width=0.99\linewidth]{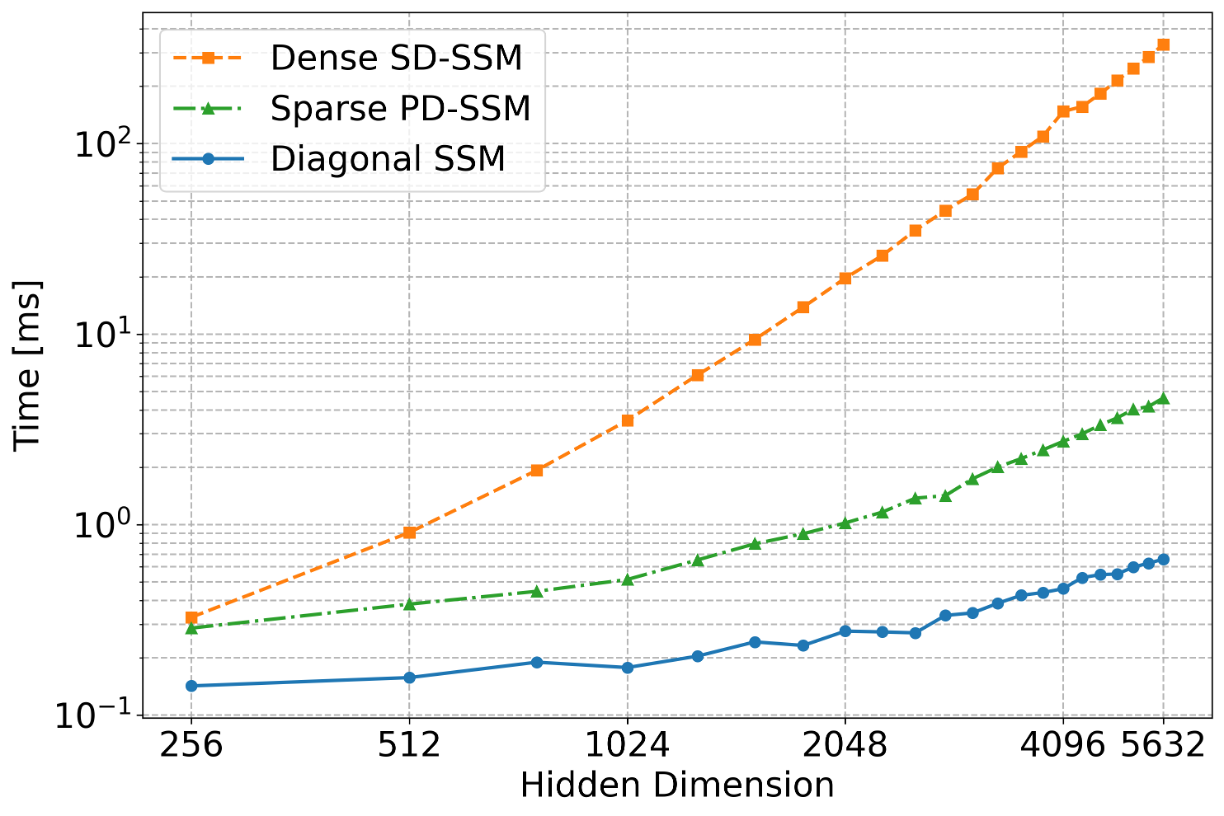}
    \captionof{figure}{Runtimes of single-layer SSMs with varying dimension and sequence length $64$.} 
    \label{fig:runtime_dimension}
\end{minipage}%

\subsection{FSA Emulation}

We first evaluate our model on a set of \emph{automaton state tracking} tasks, originally introduced in~\cite{deletang_neural_2023}. The four tasks correspond to four finite-state automata of various complexities.
The benchmark is centered at evaluating the \emph{length generalization} of the models, which serves as a proxy of the model having learned the correct algorithmic abstraction, avoiding fixed-length \emph{shortcut} solutions~\cite{liu_transformers_2023}. Concretely, the models are trained for 100,000 steps on randomly sampled sequences of inputs of length 3 to 40, and are evaluated on sequences of length 40--256. We extend the set of results from~\cite{walker_2025_structured} which evaluates each model under a single varying hyperparameter choice, state dimensionality of 128 or 512. We instead fix the dimensionality to 128, finding that it is sufficient for high performance. 
%
Table~\ref{tab:fsa-table} reports the mean and standard deviation of the best validation accuracy of five randomly initialized models. The set of evaluated models consists of recurrent and parallelizable models, where we \textbf{bold} and \uline{underline} the \textbf{best} and \uline{second-best} parallelizable model, respectively. Excluding the Transformer, the parallelizable models can be interpreted as SSMs with various transition matrix structures, namely diagonal (\cite{gu_mamba_2023, walker_2025_structured} and our $\mathbb{C}$ diagonal model defined by setting $P(u_t)$ to identity), diagonal plus low-rank (\cite{schlag2021linear,yang2024parallelizing,grazzi2024unlocking}), product of diagonal plus low-rank~\citep{walker_2025_structured,peng_2025_rwkv,siems_2025_deltaproduct}, as well as alternative variants including block-diagonal, a Walsh-Hadamard matrix modulated by an input-dependent diagonal matrix, and a mixture of diagonal and dense matrices~\cite{walker_2025_structured}. As we can read from the table, our model performs best, with a significant margin over the second-best method.

\begin{table*}[ht]
\centering
\resizebox{\textwidth}{!}{
\begin{tabular}{lccccc}
\toprule
\multicolumn{1}{l|}{\textbf{Model}} &
\textbf{Cycle Nav.} &
\textbf{Even Pairs} &
\textbf{Mod Arith.} &
\textbf{Parity} &
\multicolumn{1}{|l}{\textbf{Average}} \\
\midrule
\multicolumn{6}{l}{\textbf{Recurrent}} \\
\midrule
LSTM & 100.0 ± 0.0 & 100.0 ± 0.0 & 99.9 ± 0.1 & 100.0 ± 0.0 & 100.0 ± 0.0 \\
sLSTM & 32.5 ± 0.4 & 100.0 ± 0.0 & 27.7 ± 0.6 & 100.0 ± 0.0 & 65.1 ± 0.2 \\
xLSTM[1:1] & 53.5 ± 5.6 & 99.0 ± 1.9 & 29.3 ± 1.6 & 100.0 ± 0.0 & 70.5 ± 1.5 \\
\midrule
\multicolumn{6}{l}{\textbf{Parallel}} \\
\midrule
Transformer & 24.4 ± 0.5 & 90.4 ± 10.4 & 23.6 ± 0.7 & 52.2 ± 0.4 & 47.7 ± 2.6 \\
Mamba & 48.4 ± 2.2 & \textbf{100.0 ± 0.0} & 33.1 ± 6.6 & 54.2 ± 2.1 & 58.9 ± 1.8 \\
D-SLiCE & 69.5 ± 6.3 & \textbf{100.0 ± 0.0} & 20.9 ± 0.1 & \textbf{100.0 ± 0.0} & 72.6 ± 1.6 \\
$\mathbf{\mathbb{C}}$ \textbf{Diag.} & 90.4 ± 4.3 & 82.4 ± 4.8 & 59.9 ± 27.0 & 61.8 ± 6.8 & 73.6 ± 7.1 \\
DeltaNet & 49.8 ± 4.7 & \textbf{100.0 ± 0.0} & 42.2 ± 4.8 & 57.8 ± 0.8 & 62.5 ± 1.7 \\
DeltaNet[-1,1] & 46.7 ± 6.1 & \textbf{100.0 ± 0.0} & 66.4 ± 8.8 & 97.7 ± 2.0 & 77.7 ± 2.7 \\
Gated DeltaNet & 53.8 ± 8.8 & \textbf{100.0 ± 0.0} & 42.8 ± 8.2 & 56.5 ± 1.9 & 63.3 ± 3.0 \\
Gated DeltaProduct[-1,1] & 46.3 ± 6.6 & \textbf{100.0 ± 0.0} & 78.4 ± 10.9 & 98.0 ± 1.4 & 80.7 ± 3.2 \\
RWKV-7 & 37.8 ± 5.0 & 88.1 ± 14.2 & 39.5 ± 6.1 & 51.1 ± 0.3 & 54.1 ± 4.1 \\
DPLR-SLiCE$_{d_h=57,r=4}$ & 81.1 ± 16.6 & \textbf{100.0 ± 0.0} & 68.3 ± 19.3 & 91.0 ± 18.0 & \uline{85.1 ± 7.8} \\
WH-SLiCE & 69.7 ± 8.8 & 93.1 ± 13.9 & 23.8 ± 1.1 & 71.4 ± 12.9 & 64.5 ± 5.2 \\
BD-SLiCE$_{d_h=128,b=4}$ & \textbf{99.8 ± 0.2} & 85.9 ± 11.3 & 54.0 ± 12.5 & 95.3 ± 3.9 & 83.8 ± 4.3 \\
D-DE-SLiCE$_{d_h=272,b=16}$ & 73.3 ± 29.4 & 84.8 ± 8.5 & \textbf{98.4 ± 0.7} & 83.8 ± 11.3 & \uline{85.1 ± 8.2} \\
\textbf{PD-SSM} & \uline{99.5 ± 0.7} & \uline{99.7 ± 0.3} & \uline{96.2 ± 3.4} & \uline{99.9 ± 0.1} & \textbf{98.8 ± 0.9} \\
\midrule
Random & 20.0 & 50.0 & 20.0 & 50.0 & 35.0 \\
\bottomrule
\end{tabular}
}
\caption{Average and standard deviation of validation accuracy across 5 seeds for a range of models on FSA emulation tasks. The baseline results are taken from~\citep{walker_2025_structured}.}
\label{tab:fsa-table}
\end{table*}

%







We further analyse the method on two non-solvable groups, $A_5$ and $S_5$. For both groups, consisting of 60 and 120 permutations respectively, all of the group elements can be generated using only two permutations corresponding to two transition matrices as per the mapping in Secion~\ref{sec:background_mapping}. To increase the connectivity of the resulting automaton's states, we introduce additional randomly selected permutations.
We compare a PD-SSM with $K=32$ against two layers of the complex diagonal model with $A(u_t)=D(u_t)$, and one or two layers of a DPLR model, Gated DeltaProduct [-1,1] with $n_h=4$~\citep{siems_2025_deltaproduct}, with state dimension $128$. We perform a learning rate grid search and train the models for 100,000 steps with batch size 256. We train the models on sequences of length up to 40 and report the best validation accuracy on sequences of length up to 40-256 in Table~\ref{tab:nonsolvable}.


\begin{table*}[h!]
\centering
\resizebox{1.0\textwidth}{!}{
\begin{tabular}{lc|cccc|ccc}
\toprule
{Model} & 
{Depth} & 
\textbf{$(A_5, 2)$} & 
\textbf{$(A_5,6)$} & 
\textbf{$(A_5,8)$} & 
\textbf{$(A_5,12)$} & 
\textbf{$(S_5,4)$} &
\textbf{$(S_5,8)$} &
\textbf{$(S_5,32)$}\\
\midrule
$\mathbb{C}$ Diagonal & 2 & $15.5$ & \raisebox{0.3ex}{---} & \raisebox{0.3ex}{---} & \raisebox{0.3ex}{---} & \raisebox{0.3ex}{---} & \raisebox{0.3ex}{---} & \raisebox{0.3ex}{---} \\
Gated DP [-1,1] & 1 & $97.9$ & $92.5$ & $91.8$ & $60.5$ & $88.7$ & $57.4$  & $1.23$ \\
Gated DP [-1,1]& 2 & \raisebox{0.3ex}{---} & \raisebox{0.3ex}{---} & \raisebox{0.3ex}{---} & $68.4$ & $84.7$ & $62.2$ & \raisebox{0.3ex}{---}  \\
\textbf{PD-SSM} & 1 & $100.0$ & $100.0$ & $100.0$ & $100.0$ & $100.0$ & 100.0 & $1.07$  \\
\bottomrule
\end{tabular}
}
\caption{Best validation accuracy (\%) on longer sequences across 3 random seeds. Each task is defined in terms of a non-solvable algebraic group ($A_5$ or $S_5$) generated with redundant set of permutations, the cardinality of which is the number $n$ in each tuple $(A_5 \text{ or }S_5,n)$.}
\label{tab:nonsolvable}
\end{table*}

As predicted by the theory, the complex diagonal model fails on the simplest task, $(A_5,2)$. Gated DeltaProduct with $n_h=4$ exhibits degrading accuracy with a higher number of permutations, which cannot be recovered by stacking two layers of the model. \name{} maintains full accuracy for all automata with a single layer, excepting $S_5$ with $32$ generating permutations.

\subsection{Multivariate Time-Series Classification}

We next provide an evaluation on multivariate time-series classification. We evaluate our model on a subset of the University of East Anglia (UEA) Multivariate Time-Series Classification Archive (UEA-MTSCA)~\citep{uea_time_series}, extending on the results from~\cite{walker_2024_logncde, rusch2025oscillatory}. We consider six tasks from the archive previously selected due to their long sequence lengths, which range from around 400 to over 17,000. Conforming to the evaluation methodology and data splits defined in~\cite{rusch2025oscillatory, walker_2024_logncde}, we report the average and standard deviation of the test accuracy across five random initializations, with the hyperparameter grid conforming to that of~\citep{rusch2025oscillatory} with the exception of the state size, which in our case is selected from $\{16,64,128\}$ as opposed to $\{16,64,256\}$ of the baselines. The set of evaluated models includes neural controlled differential equations~\citep{walker_2024_logncde}, LTI SSMs~\citep{smith_simplified_2023,orvieto_resurrecting_2023}, time-varying SSMs~\citep{gu_mamba_2023}, as well as an SSM paradigm based on a system of forced harmonic oscillators~\citep{rusch2025oscillatory}. The results are reported in Table~\ref{tab:time_series}. As we can see, our model maintains very high accuracy on this set of tasks, achieving an estimated mean accuracy that is within standard error of the state-of-the art SSM.

\begin{table*}[h!]
\centering
\resizebox{\textwidth}{!}{
\begin{tabular}{lcccccc|c}
\toprule
{\textbf{Model}} & 
\textbf{Worms} & 
\textbf{SCP1} & 
\textbf{SCP2} & 
\textbf{Ethanol} & 
\textbf{Heartbeat} &
\textbf{Motor} &
{\textbf{Average}} \\
\midrule
NRDE & 83.9 ± 7.3 & 80.9 ± 2.5 & 53.7 ± 6.9 & 25.3 ± 1.8 & 72.9 ± 4.8 & 47.0 ± 5.7 & 60.6 ± 2.15 \\
NCDE & 75.0 ± 3.9 & 79.8 ± 5.6 & 53.0 ± 2.8 & 29.9 ± 6.5 & 73.9 ± 2.6 & 49.5 ± 2.8 & 60.2 ± 1.76 \\
Log-NCDE & 85.6 ± 5.1 & 83.1 ± 2.8 & 53.7 ± 4.1 & \uline{34.4 ± 6.4} & 75.2 ± 4.6 & 53.7 ± 5.3 & 64.3 ± 1.99 \\
LRU & {87.8 ± 2.8} & 82.6 ± 3.4 & 51.2 ± 3.6 & 21.5 ± 2.1 & \uline{78.4 ± 6.7} & 48.4 ± 5.0 & 61.7 ± 1.72 \\
S5 & 81.1 ± 3.7 & \textbf{89.9 ± 4.6} & 50.5 ± 2.6 & 24.1 ± 4.3 & {77.7 ± 5.5} & 47.7 ± 5.5 & 61.8 ± 1.83 \\
S6 & 85.0 ± 16.1 & 82.8 ± 2.7 & 49.9 ± 9.4 & 26.4 ± 6.4 & 76.5 ± 8.3 & 51.3 ± 4.7 & 62.0 ± 3.68 \\
Mamba & 70.9 ± 15.8 & 80.7 ± 1.4 & 48.2 ± 3.9 & 27.9 ± 4.5 & 76.2 ± 3.8 & 47.7 ± 4.5 & 58.6 ± 2.99 \\
LinOSS-IMEX  & 80.0 ± 2.7 & {87.5 ± 4.0} & \textbf{58.9 ± 8.1} & 29.9 ± 1.0 & 75.5 ± 4.3 & \uline{57.9 ± 5.3} & {65.0 ± 1.95} \\
LinOSS-IM & \textbf{95.0 ± 4.4} & \uline{87.8 ± 2.6} & \uline{58.2 ± 6.9} & {29.9 ± 0.6} & 75.8 ± 3.7 & \textbf{60.0 ± 7.5} & \textbf{67.8 ± 2.00} \\
\textbf{PD-SSM} & \uline{90.0 ± 5.7} & 80.9 ± 2.0 & {56.1 ± 8.6} & \textbf{34.7 ± 4.0} & \textbf{80.0 ± 2.6} & \textbf{60.0 ± 3.7} & \uline{67.0 ± 2.02} \\
\bottomrule
\end{tabular}
}
\caption{Mean and standard deviation of test accuracies across 5 seeds on selected long-sequence UEA time-series classification datasets as per~\citep{rusch2025oscillatory}.}
\label{tab:time_series}
\end{table*}

\subsection{Long-Range Arena}

%




The long-range arena~\citep{tay2020long} dataset (LRA) covers mathematical expressions, freeform text, and images, all represented in sequences up to length 16k. 
We evaluate our model on a subset of LRA, with sequences up to length 4k.
The baseline results are taken from~\cite{soydan_2024_simplified}.
\normalsize
We consistently use 4 layers with embedding dimension 128 and state dimension 128 for all tasks except 
%
%
\noindent
\begin{minipage}{0.41\linewidth}
\emph{Retrieval}, which used state dimension 64. For all tasks, the transition matrix dictionary size is set to $K=6$.
Time-invariant SSMs perform significantly better on average than the shown collection of time-variant ones.
Among the time-variant SSMs (Mamba, S7, and \name), ours performs best on average on this set of tasks.
Together with the time-series results, this serves as evidence that the PD parametrization can be effective on realistic long-sequence tasks.
%



\end{minipage}%
\hfill
\begin{minipage}{0.56\linewidth}
\centering
\small
\begin{tabularx}{1\linewidth}{rcc!{\vrule width 1pt}ccc}
\toprule
 & \multicolumn{2}{c}{Time-Invariant}  & \multicolumn{3}{c}{Time-Variant} \\
\cmidrule(r){2-3} \cmidrule(r){4-6} 
Dataset          
& S4 & LRU & Mamba & S7 & \name \\ 
\cmidrule(r){1-1} \cmidrule(r){2-2} \cmidrule(r){3-3} \cmidrule(r){4-4} \cmidrule(r){5-5} \cmidrule(r){6-6}

ListOps     
& 59.6 & 60.2 & 38.0 & \textbf{63.8} & \uline{61.0}  \\ 

Text     
& 86.8 & {89.4} & 83.0 & \uline{87.2} & \textbf{88.1}   \\ 

Image    
& 88.6 & 89.0 & \uline{69.8} & 61.1 & \textbf{70.4}  \\

Retrieval    
& 90.9 & 89.9 & 72.1 & \textbf{91.8} & \uline{90.0}  \\

Pathfinder    
& {94.2} & 95.1 & \textbf{69.3} & \uline{65.5} & 62.6  \\

\cmidrule(r){1-1} \cmidrule(r){2-2} \cmidrule(r){3-3} \cmidrule(r){4-4} \cmidrule(r){5-5} \cmidrule(r){6-6}
Average     
& 84.0 & {84.7} & 66.4 & \uline{73.9} & \textbf{74.4} \\
\bottomrule       
\end{tabularx}
\captionof{table}{LRA results, average over 3 seeds.}
\label{tab:lra}
\end{minipage}%






\subsection{State Tracking in Natural Language}\label{sec:experiments_state_tracking_in_natural_language}

%
%
%
\noindent
\begin{minipage}{0.41\linewidth}
\centering
\includegraphics[width=\textwidth]{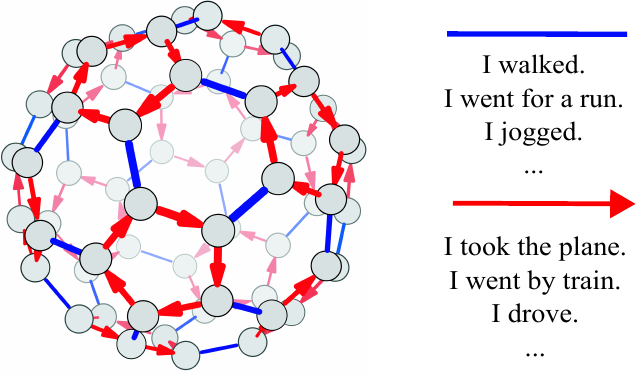}
\captionof{figure}{Cayley diagram of the $A_5$ group~\citep{carter_visual_2009}. We encode state transitions into sets of English sentences.}
\label{fig:A5_NL}
\end{minipage}%
\hfill
\begin{minipage}{0.56\linewidth}
  Finally, we introduce a novel task which we call \emph{state tracking in natural language}, a more complex version of the FSA state tracking task with the crucial difference that the inputs are encoded in natural language. As a result, a state transition is only triggered after a length-varying sequence of tokens. A  real-world instance of such state tracking in natural language is geographic location tracking given text information, where the model has to work with descriptions such as \textit{I took the bus}, \textit{I walked}, and \textit{I went for a run}. In this benchmark, we assume that an agent transitions through a fixed state space according to a sequence of meaningful English sentences. Each transition in the underlying automaton is redundantly encoded through multiple sentences of different lengths. For instance, $A_5$ is understood as enc- 
\end{minipage}%
%
%
%

%
\vspace{-0.5em}
oding a transportation network where an agent can either move along a ring of motorized transportation (via \textit{bus, plane, train, car}) or move (i.e., \textit{walk, run, jog, hike}) between such rings. \textcolor{red}{Red} transitions correspond to movement using motorized transport, whereas \textcolor{blue}{blue} transitions are taken without motorized transport, see Figure~\ref{fig:A5_NL}. Figure~\ref{fig:nl_state_tracking} shows the performance of diagonal SSMs as well as \name{} for state tracking in natural language. To benefit from large-scale pretraining and to demonstrate the modularity of our approach, we freeze a pretrained Qwen 2.5-1.5B model and replace its final layer with a single trainable SSM layer. This analysis centers around the applicability of the sparse parametrization in a larger-scale setting, and also serves as a test of the utility of the $PD$ parametrization as opposed to purely diagonal models. Concretely, we vary the matrix structure such that it is either real-valued diagonal ($A(u_t)=|D(u_t)|$), complex-valued diagonal ($A(u_t)=D(u_t)$), or structured sparse ($A(u_t)=P(u_t)D(u_t)$). We train the model for 100,000 steps with batches of size 256. Each batch consists of sequences of English sentences, with the number of sentences randomly sampled up to 25, each sentence triggering a state transition in the underlying automaton. The automata are Parity and $A_5$ with two generators as in Figure~\ref{fig:A5_NL}. In Figure~\ref{fig:nl_state_tracking} we report the best validation performance out of 3 random seeds of each model with an equivalent hyperparameter grid search.
Even though complex diagonal SSMs can represent parity~\citep{sarrof2024expressivecapacitystatespace}, the models did not learn to do so in the experiments. On this task, \name{} converges within 10,000 steps.

\begin{figure}[h]
    \centering
    \begin{subfigure}[b]{0.49\textwidth}
        \centering
        \includegraphics[width=\textwidth]{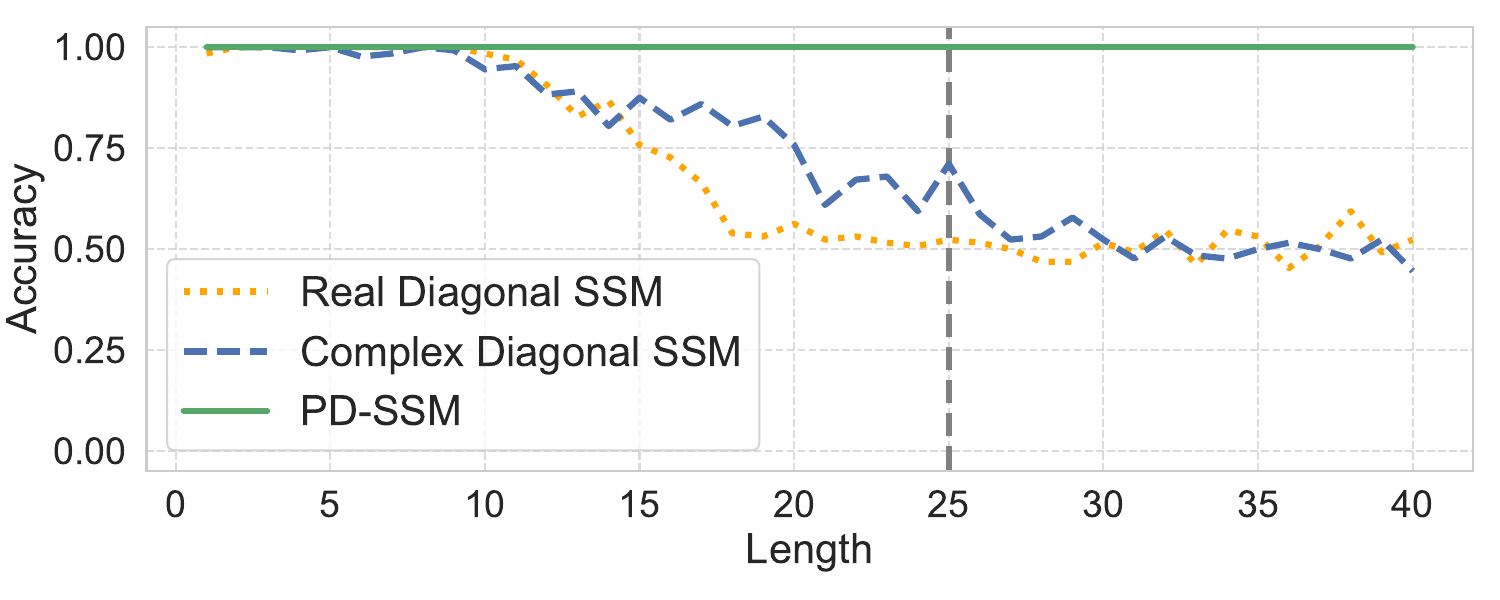}
        \caption{Accuracy on the natural language version of Parity}
    \end{subfigure}
    \hfill
    \begin{subfigure}[b]{0.49\textwidth}
        \centering
        \includegraphics[width=\textwidth]{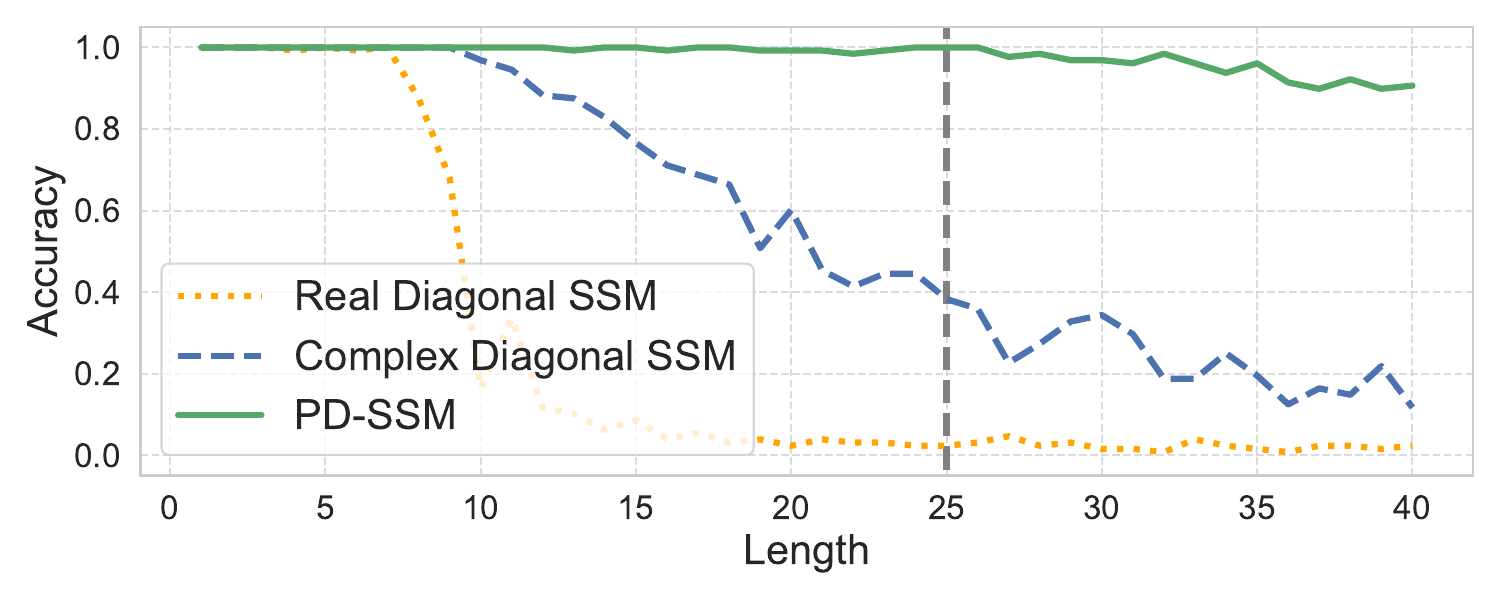}
        \caption{Accuracy on the natural language version of $A_5$}
    \end{subfigure}
    \caption{Applying one layer of our PD-SSM at the output of a frozen Qwen-2.5 1.5B allows the model to learn to follow the state transitions of automata when the inputs are redundantly encoded in English sentences. The vertical dashed line indicates the maximum training sequence length.
    }
    \label{fig:nl_state_tracking}
\end{figure}
%

\section{Conclusion}
We introduced an expressive and efficient structured sparse parametrization of transition matrices for time-varying SSMs manifesting in the \name{} architecture. Theoretically, \name{} can emulate arbitrary finite-state automata while reaching the lower bound on the network depth and state size required to do so in the worst case. It is significantly more efficient than unstructured SSMs and reaches a new state-of-the-art in synthetic state-tracking tasks. 
Among the investigated time-varying SSMs, \name on average performs best on tasks from the LRA benchmark requiring the processing of long sequences (up to 4k), and additionally exhibits high performance on multivariate time-series classification with long sequences (up to 17k), demonstrating performance within standard error of the state-of-the-art SSM.
Finally, we verified its effectiveness in a hybrid Attention-SSM architecture on a novel state-tracking in natural language task, showing that it can track non-solvable automaton states even when the inputs are redundantly encoded in variable-length English sentences.

\paragraph{Limitations and Future Work} Although we significantly improve upon the scalability of SD-SSM while retaining its favorable expressivity, \name{} still incurs overhead when generating the $P(u_t)$ matrices.
Our future work will investigate how our method can benefit from more efficient one-hot column matrix generation strategies, and how it can be utilized in large-scale pretraining. 
%

\section{Broader Impact}
The paper introduces a novel neural network architecture with increased expressivity while retaining efficiency. More expressive and efficient models can potentially have unpredictable societal impacts in the future, but we do not foresee any immediate and direct negative impact of this work.

\section*{Acknowledgement} 

This work is supported by the Swiss National Science foundation (SNF), grant 10002666.

\bibliography{bibliography}  

\newpage

\appendix

\setcounter{figure}{0}
\renewcommand{\thefigure}{S\arabic{figure}}
\setcounter{table}{0}
\renewcommand{\thetable}{S\arabic{table}}


\section{Background}\label{app:background}

The goal of this section is to provide knowledge of abstract algebra and circuit complexity theory at a level sufficient to understand the motivation behind our paper. We omit certain details for reasons of clarity, and we often focus on providing an intuitive understanding by considering visual and textual explanations and examples for important concepts. This approach is highly inspired by~\cite{carter_visual_2009}. As the core of our work is the design and analysis of a novel neural network architecture, one does not need to have a deep understanding of these topics to understand our paper. As such, the entire Appendix A can be skipped without significantly affecting the understanding of our PD method. For a deeper treatment of the topics, one may consult the following resources:~\citep{straubing_book,carter_visual_2009,arora_barak_2009}.

We start by motivating the connection between abstract algebra and FSAs in~\ref{app:algebra-fsm}. We then explain the central algebraic concepts in~\ref{app:groups}, starting with the definition of a group and finishing with step-by-step reasoning about why a certain group is solvable. In~\ref{app:circuit-complexity}, we introduce the main relevant concepts and classes from circuit complexity theory. We finish in~\ref{app:ssm-expressivity} by examining the most important results on the expressivity of SSMs that combine the two presented frameworks.

\subsection{Algebra and Finite-State Automata}\label{app:algebra-fsm}

Algebra is, broadly speaking, the study of \emph{sets} and \emph{binary operations} on elements of sets. The set elements can be of various nature. An intuitive example of an algebraic structure is the set of integers under addition, often denoted as $\mathbb{Z}^+$. In more generality, one may consider sets of functions with the binary operation being function composition. This is in fact highly related to the study of finite-state automata (FSAs), where we consider \emph{sets of transition functions corresponding to the input elements}, and we study the effect of \emph{composing transition functions}. This described set is defined for any automaton $\mathcal{A}=(Q,\Sigma,\delta)$ as
\begin{align*}
    \mathcal{T}(\mathcal{A)} = \{ \delta(\cdot,\sigma_{1,\dots,T}):T\in\mathbb{N}, \forall i \in [T], \sigma_i \in \Sigma\}
\end{align*}
where $\delta(\cdot,\sigma_{1,\dots,T}):Q\rightarrow Q$ is the transition function obtained by composing the transition functions of each input in the sequence, namely, $\delta(\cdot,\sigma_{1,\dots,T})=\delta(\delta(\dots\delta(\cdot,\sigma_1),\dots\sigma_{T-1}),\sigma_T)$.

This set contains all of the transition functions that can be expressed by the given automaton.
The limitations of SSMs on FSA emulation stem from the discrepancy between $\mathcal{T}(\mathcal{A)}$ and the transition functions that SSMs with different transition matrix structures can express.
A motivating illustrative example is given in Figure~\ref{fig:algebra_bg}, in which we consider the $A_5$ automaton and a time-variant SSM with diagonal transition matrices and $\forall t:b(u_t)=0$. Please note that this figure only serves to illuminate the connection between FSAs with algebraic structure and SSMs, with precise bounds being presented later.
\begin{figure}[ht]
    \centering
    \includegraphics[width=1.0\linewidth]{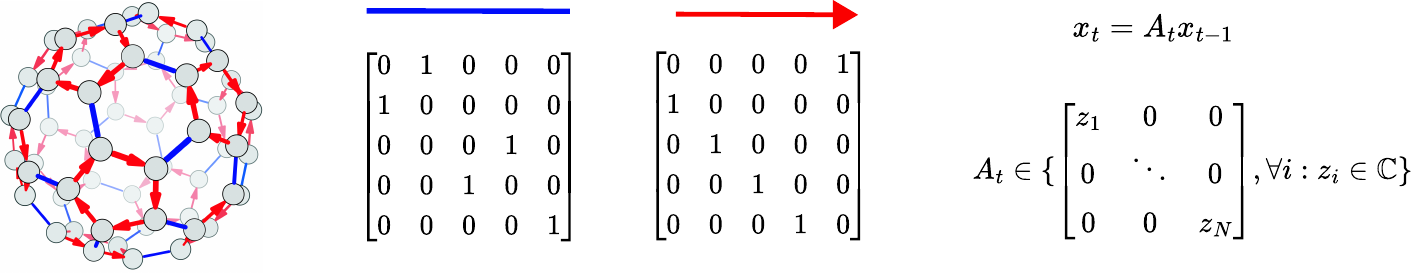}
    \caption{\emph{Left:} A visualization of the $A_5$ automaton as we used it in our experiments~\citep{carter_visual_2009}. \emph{Center:} The two transitions of $A_5$ equivalently represented in matrix form. $\mathcal{T}(\mathcal{A})$ is isomorphic to the set of all possible unique iterated products of the two matrices, of which there are a total of 60. \emph{Right:} A simplified diagonal SSM that omits the $B(u_t)u_t$ term. As the transition matrices are restricted to be diagonal, and since the two transition matrices represented in the center are not simultaneously diagonalizable, there seems to exist a discrepancy between $\mathcal{T}(\mathcal{A})$ and the SSM.}
    \label{fig:algebra_bg}
\end{figure}

The reasoning in the figure above is only valid for a single SSM layer with no $B(u_t)u_t$ term. In order to reason about several layers of the full model, recent work~\cite{merrill_illusion_2024, sarrof2024expressivecapacitystatespace} have utilized the frameworks of \emph{Abstract Algebra} (A.2) and \emph{Circuit Complexity Theory} (A.3).

\subsection{Algebraic Groups and their Properties}\label{app:groups}

We will now define algebraic objects that are of interest to us. We start with \emph{groups}.
%
\begin{defn}[Algebraic Group]
An \emph{algebraic group} is a set $G$ equipped with a binary operation $\cdot$ that is
\begin{enumerate}[noitemsep, topsep=0pt]
  \item \textbf{Associative:} \((a\cdot b)\cdot c = a\cdot(b\cdot c)\) for all \(a, b, c \in G\),
  \item Has a \textbf{neutral element} \(e \in G\) such that \(a\cdot e = e \cdot a = a\) for all \(a \in G\),
  \item And where every element has an \textbf{inverse}: for each \(a \in G\), there exists \(a^{-1} \in G\) such that \(aa^{-1} = a^{-1}a = e\).
\end{enumerate}
\end{defn}

\noindent
\begin{minipage}{0.80\linewidth}

Algebraic groups can be graphically represented using \emph{Cayley diagrams}. Starting from the neutral element, a Cayley diagram visualizes the structure of the group by applying the group operation in conjunction with a group element to the current element. This can be easily understood by considering the group of integers under addition modulo 4 ($\mathbb{Z}_4$) in the figure on the right. We draw the Cayley diagram by first drawing the neutral element, $0$. Starting from $0$, all other elements are generated by iteratively applying $+1$ on the current element. For
\end{minipage}%
\hfill
\begin{minipage}{0.15\linewidth}
\centering
\includegraphics[width=\textwidth]{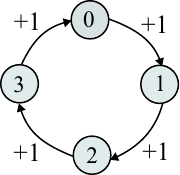}
\end{minipage}%

some groups, in order to generate all elements of the  set, we need to consider more than one generating element. An example is the $A_5$ group. While Figure~\ref{fig:algebra_bg} refers to it as an \emph{automaton}, the figure is equivalently the Cayley diagram of $A_5$ with the two selected generating permutations. 

A central property of groups is the \emph{commutativity of the binary operation}. The binary operation of a group is commutative only if $\forall a,b \in G:a\cdot b=b\cdot a$.

\begin{figure}[ht]
    \centering
    \includegraphics[width=0.45\linewidth]{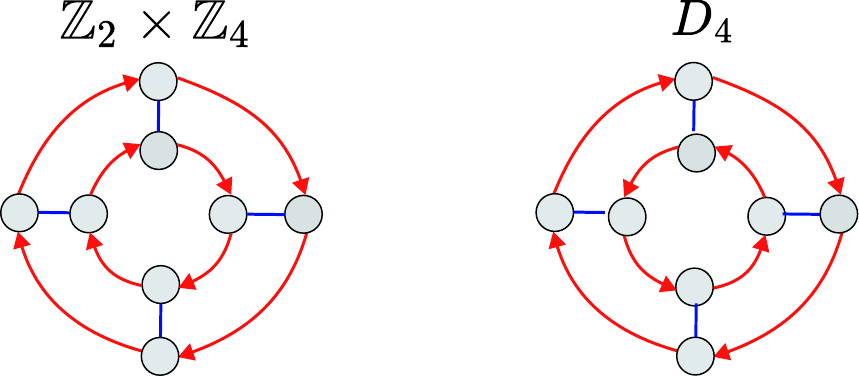}
    \caption{Cayley diagrams of the group $\mathbb{Z}_2\times \mathbb{Z}_4$ and the dihedral group $D_4$~\citep{carter_visual_2009}. $\mathbb{Z}_2 \times \mathbb{Z}_4$ is a commutative group, while $D_4$ is not. Both are solvable, with solvability of $\mathbb{Z}_2 \times \mathbb{Z}_4$ shown in Example~\ref{ex:solvable}.} 
\end{figure}

\paragraph{Solvability} Another central property of groups is \emph{solvability}. We will first present it informally and will then precisely define it. Loosely speaking, a group is \emph{solvable} if it can be constructed from co-
\noindent
\begin{minipage}{0.38\linewidth}
\centering
\includegraphics[width=\textwidth]{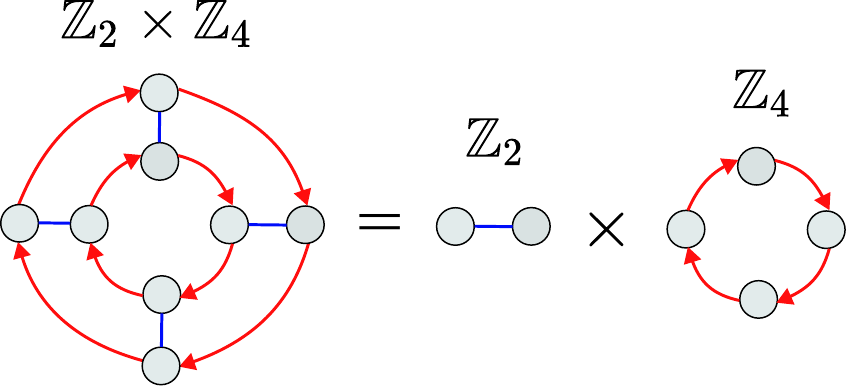}
\captionof{figure}{Composition of $\mathbb{Z}_2 \times \mathbb{Z}_4$.}
\label{fig:C2_C4_decomposition}
\end{minipage}%
\hfill
\begin{minipage}{0.58\linewidth}
 mmutative groups. The details of the construction are particularly easy to understand on the example of the $\mathbb{Z}_2\times \mathbb{Z}_4$ group, shown in Figure~\ref{fig:C2_C4_decomposition}. The $\mathbb{Z}_2\times \mathbb{Z}_4$ group is solvable, as it corresponds to the direct product of two commutative groups. While the $D_4$ group shown previously is not commutative, it is solvable as it can be constructed from commutative groups in a similar manner as $\mathbb{Z}_2 \times \mathbb{Z}_4$. However, $A_5$ cannot be constructed using commutative groups~\cite{carter_visual_2009}. To speak about solvability in more definite terms, we need to define several concepts. The re-
\end{minipage}
%
%
mainder of this subsection defines and explains solvability in concrete terms.
Firstly, we will define subgroups, which are in essence closed subsets of a group.
\begin{defn}[Subgroup]
Let \(G\) be a group . A subset \(H \subseteq G\) is called a \emph{subgroup} if \(H\) is itself a group under the operation inherited from \(G\). That is, \(H\) is a subgroup if:
\begin{enumerate}[noitemsep, topsep=0pt]
  \item \(e \in H\), where \(e\) is the identity element of \(G\),
  \item For all \(a, b \in H\), the product $a \cdot b$ is in $H$,
  \item For all \(a \in H\), the inverse $a^{-1}$ is in $H$.
\end{enumerate}
\end{defn}

Of central interest are \emph{normal subgroups}, which is a particularly well-behaved family of subgroup. We will use normal subgroups to partition a group into a sequence of subgroups. The solvability of a group can be determined by examining the structure of a sequence of normal subgroups.
\begin{defn}[Normal Subgroup]\label{def:normal_subgroup}
A \emph{normal subgroup} \(N\) of a group \(G\), written \(N\triangleleft G\), is a subgroup invariant under conjugation by elements of \(G\), i.e., for each \(n\in N\), it holds that \(gng^{-1} = n\) for all \(g \in G\).
\end{defn}
In particular, the sequence of group partitions we want to examine is the group's \emph{normal series}.
\begin{defn}[Normal Series]
    A \emph{normal series} of a group \(G\) is a finite sequence of subgroups
\[
\{e\} = G_0 \triangleleft G_1 \triangleleft \cdots \triangleleft G_n = G,
\]
such that each subgroup \(G_i\) is normal in the next one, i.e., \(G_i \triangleleft G_{i+1}\) for all \(i\), and the series terminates in the trivial group $\{e\}$.
\end{defn}

Apart from subgroups, we need to define \emph{cosets}, as they are a central concept when analyzing how groups decompose. A coset is a subset of the group formed by a subgroup, but in contrast to a subgroup, it is not necessarily closed under the group operation.
\begin{defn}[Coset]
Let \(G\) be a group and \(H\) a subgroup of \(G\). For any element \(g \in G\), the \emph{left coset} of \(H\) in \(G\) with respect to \(g\) is the set
\[
gH := \{gh : h \in H\}.
\]
Similarly, the \emph{right coset} is
\[
Hg := \{hg : h \in H\}.
\]

Cosets partition the group \(G\) into disjoint subsets of equal size. 
\end{defn}
Consider, for example, the partition of the integers into even and odd. As the following example shows, the even numbers are a normal subgroup, the odd numbers are a coset, and together they form all integers.

\begin{example}
    The integers under addition form a group which we write as \(\mathbb{Z^+}\). 
    The even integers, written as \(2\mathbb{Z^+}\), form a subgroup of \(\mathbb{Z^+}\), as the sum of two even numbers is again even and the neutral element, 0, is even. 
    In particular, this is a normal subgroup, an easily verifiable immediate consequence of the commutativity of addition.
    
    The odd integers are a coset of \(\mathbb{Z^+}\) constructed from \(2\mathbb{Z}^+\) by adding 1 to each even number. That is, the odd integers form the set \(1+2\mathbb{Z}^+\). This is not a subgroup, as it does not contain the neutral element.
    The sets are disjoint, are of equal size by a 1-to-1 mapping of the elements, and together they form \(\mathbb{Z}\).
\end{example}

It is precisely the decomposition of a group by use of a normal subgroup that we are interested in. This procedure creates a new group structure called a \emph{factor group}, which describes the relation between a group and its normal subgroup.

\begin{defn}[Factor Group]\label{def:factor_group}
Let $N$ be a normal subgroup of a group $G$. A \emph{factor group} (or quotient group) \(G/N\) is the set of left cosets of \(N\) in \(G\), concretely $(G / N=\{gN:g\in G\}$, equipped with a natural group operation.
\end{defn}

Let us again consider the even and odd integers. We have seen that this is the full set of cosets of \(\mathbb{Z^+}\). If we abstract away the individual numbers and only consider the concepts of \emph{even} and \emph{odd}, we can see that the two concepts form a group structure, with the natural operation defined as per the sketch below.

\begin{example}\label{ex:2}
    Consider the set \( \{\text{Even},\text{Odd}\}\) with the binary operation defined as per Figure~\ref{fig:even_odd}. This figure is in fact the Cayley diagram of~\(\mathbb{Z}/2\mathbb{Z}\).
    
    \begin{figure}[ht]
    \centering
    \includegraphics[width=0.5\linewidth]{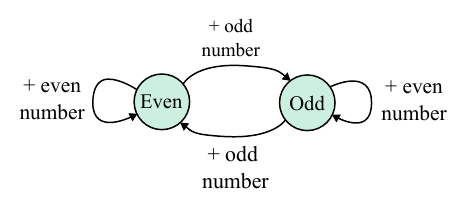}
    \caption{The \emph{Even} and \emph{Odd} cosets of \(\mathbb{Z}\) themselves form a group under addition, in particular the factor group \(\mathbb{Z}/2\mathbb{Z}\). The even numbers act as the neutral element, addition is associative, and the integers are closed under addition.} 
    \label{fig:even_odd}
\end{figure}

\end{example}

A solvable group is one in which we can create a sequence of such factor groups with each factor group being commutative. The sequence should also end with the trivial group containing only the neutral element.
\begin{defn}[Solvable Group]
A finite group \(G\) is \emph{solvable} if there exists a finite sequence of subgroups
\[
\{e\} = G_0 \triangleleft G_1 \triangleleft \cdots \triangleleft G_n = G
\]
such that each \(G_i\) is a normal subgroup of \(G_{i+1}\) and each factor group \(G_{i+1}/G_i\) is commutative.
Equivalently, \(G\) has a normal series whose successive quotients are commutative groups.
\end{defn}

We will understand these concepts better if we use them, for example, to argue about why $\mathbb{Z}_2 \times \mathbb{Z}_4$ is a solvable group.

\begin{example}\label{ex:solvable}

Consider the group \(G = \mathbb{Z}_2 \times \mathbb{Z}_4\). We can understand the group by considering its Cayley diagram visualized in Figure~\ref{fig:C24}:

\begin{figure}[h!]
    \centering
    \includegraphics[width=0.3\linewidth]{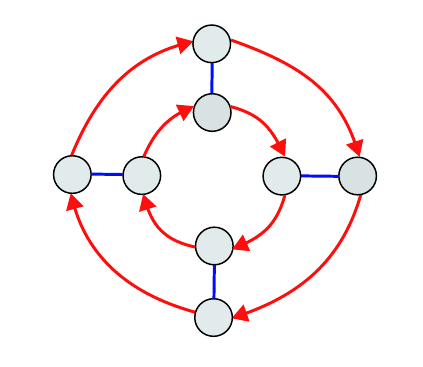}
    \caption{The $\mathbb{Z}_2 \times \mathbb{Z}_4$ group Cayley diagram.} 
    \label{fig:C24}
\end{figure}

We can identify the elements \(G = \mathbb{Z}_2 \times \mathbb{Z}_4\) using two integers \([a,b]\). We define the group operation as \([a,b]+[c,d]=[a+c \text{ mod } 2, b+d\text{ mod } 4]\). If we choose the upper-most element to be the neutral element \([0,0]\), and generate the rest of the group using \(+[0,1]\) and \(+[1,0]\), we obtain the Cayley diagram shown below.

\begin{figure}[h!]
    \centering
    \includegraphics[width=0.4\linewidth]{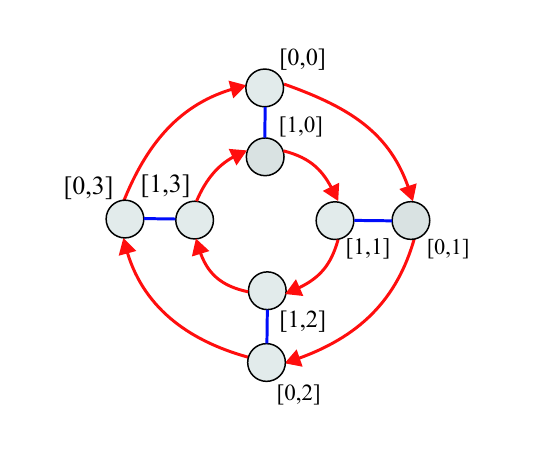}
    \caption{The $\mathbb{Z}_2 \times \mathbb{Z}_4$ group Cayley diagram, labeled by integer tuples.} 
    \label{fig:C24_labeled}
\end{figure}

We can first notice that the outer ring in the diagram forms a subgroup. Adding any two elements on this ring produces an element on that same ring, and the neutral element \([0,0]\) is part of it. This group is isomorphic to the integers under addition modulo 4, written as \(\mathbb{Z}_4\).
The outer ring is in fact a normal subgroup of \(\mathbb{Z}_2 \times \mathbb{Z}_4\). When the group operation is commutative, as it is in this case, conjugation is a trivial operation and each subgroup is normal. 
Concretely, if we use commutativity in the definition of a normal subgroup (Definition~\ref{def:normal_subgroup}), we see that \(gng^{-1}=gg^{-1}n=n\) trivially holds. 

The outer ring forms a normal subgroup isomorphic to $\mathbb{Z}_4$, and we know that we can use a normal subgroup to construct a factor group as per Definition~\ref{def:factor_group}. We have seen this in Example~\ref{ex:2}, when we created a factor group dividing all integers using the even ones. 
In this case of partitioning \(G = \mathbb{Z}_2 \times \mathbb{Z}_4\) using \(\mathbb{Z}_4\), we obtain the same factor group as we did in Example 2. 
To this end, we will generate the coset \([1,0]+[0,\mathbb{Z}_4]\), which contains all elements of the inner ring in Figure~\ref{fig:C24_labeled}. It is not a subgroup, as it does not contain the neutral element. 
We can easily verify that the two cosets of \(\mathbb{Z}_2 \times \mathbb{Z}_4\) follow the behavior shown in Figure~\ref{fig:rings}.

\begin{figure}[ht]
    \centering
    \includegraphics[width=0.65\linewidth]{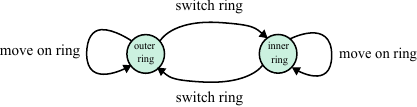}
    \caption{The factor group obtained by dividing $\mathbb{Z}_2 \times \mathbb{Z}_4$ by the outer ring normal subgroup.} 
    \label{fig:rings}
\end{figure}

Abstracting away the details, this is conceptually the same figure as the one we have seen previously when dividing \(\mathbb{Z}\) by \(2\mathbb{Z}\).
The factor group is therefore \(\mathbb{Z}/2\mathbb{Z}\). We have now constructed the first part of the normal series, \(\mathbb{Z}_4 \triangleleft \mathbb{Z}_2 \times \mathbb{Z}_4\), and we know that the factor group \(\mathbb{Z}_2\times \mathbb{Z}_4/\mathbb{Z}_4 =\mathbb{Z}/2\mathbb{Z}\) is commutative because its group operation is isomorphic to integer addition modulo 2. 

We can now perform a similar exercise to derive that \(\mathbb{Z}_2 \triangleleft \mathbb{Z}_4\) and that \(\mathbb{Z}_4/\mathbb{Z}_2\) is a commutative group. 
However, we spare ourselves the effort by first noting that \(\{e\}\triangleleft \mathbb{Z}_4\) and that \(\mathbb{Z}_4/\{e\}\) is commutative. Firstly, the trivial group is a normal subgroup of any group. Secondly, \(\mathbb{Z}_4/\{e\}\) is commutative because \(\mathbb{Z}_4\) is commutative and for each group \(G\), \(G/\{e\}=G\). Both of these properties can be verified using the tools we introduced above.
With this, we complete the chain of normal subgroups with commutative factor groups, showing the solvability of \(\mathbb{Z}_2 \times \mathbb{Z}_4\) by 
\[
\{e\} \triangleleft \mathbb{Z}_4 \triangleleft \mathbb{Z}_2 \times \mathbb{Z}_4
\]

The entire exercise could of course have been skipped by directly noting that \(\mathbb{Z}_2 \times \mathbb{Z}_4\) is commutative. In this case, trivially stating \(\{e\} \triangleleft \mathbb{Z}_2 \times \mathbb{Z}_4\) suffices.
\end{example}

We are now equipped with a basic understanding of what a solvable group is. To better understand the implications of this concept for state-tracking with SSMs, we need to introduce a different, yet perhaps surprisingly related framework~\citep{straubing_book}, that of \emph{Boolean circuit complexity}.

\subsection{Circuit Complexity Theory}\label{app:circuit-complexity}

 The theory of circuit complexity aims to provide an abstract understanding of highly parallel algorithms. One might imagine that there exist problems for which it is difficult, if not impossible, to design one efficient algorithm which operates on arbitrary-length inputs, yet in which all input lengths admit a bespoke, highly efficient and parallel algorithm\footnote{As per~\citep{arora_barak_2009}, introductory remark to Chapter 6, Karp and Lipton propose a variant of this idea in 1982.}. Such considerations are especially important in cryptography, as an efficient algorithm for the factorization of n-bit integers would be a significant development for any large $n$. Boolean circuit complexity is the study of such parallel algorithms, or rather \emph{circuit families}. Let us first define a Boolean circuit.

\begin{defn}[Boolean Circuit]
   A \emph{Boolean circuit} is a directed acyclic graph where internal nodes are logic gates (typically AND, OR, and NOT), leaves are input variables or constants ($0$ or $1$), and there is a single designated output node. The circuit computes a Boolean function by propagating values from the inputs through the gates to the output.
\end{defn}

A Boolean circuit family is simply the collection of problem-specific Boolean circuits, with one circuit being defined per input length.

\begin{defn}[Boolean Circuit Family]
   A {Boolean circuit family} is a sequence of Boolean circuits $\{C_n\}_{n \in \mathbb{N}}$, where each circuit $C_n$ has $n$ input variables and produces a single output bit.
\end{defn}

We will now present the two classes of circuits directly relevant to our work.

\begin{defn}[$TC^i$ Circuit Complexity Classes]
   The class $\mathsf{TC}^i$ consists of languages decidable by families of Boolean circuits of polynomial size $O(poly(n))$ and depth $O(\log^i n)$, with unbounded fan-in AND, OR, NOT, and majority (threshold) gates.
\end{defn}

$TC^0$ consists of shallow (constant-depth) but arbitrary fan-in Boolean circuits of standard elements (AND, OR, NOT) augmented with a majority gate. Emulating solvable automata is proven to be within $TC^0$~\citep{sarrof2024expressivecapacitystatespace}. 

\begin{defn}[$NC^1$ Circuit Complexity Class]
   The class $\mathsf{NC}^1$ consists of languages decidable by families of Boolean circuits of polynomial size $O(poly(n))$ and depth $O(\log n)$, with AND, OR, and NOT gates of fan-in 2.
\end{defn}

$NC^1$ allows for logarithmic depth, but restricts the fan-in of the gates to 2. Emulating solvable automata is proven to be an $NC^1$-complete problem~\cite{barrington_1989}, meaning that all $NC^1$ problems can be reduced to it.

Many relations between circuit complexity classes remain unknown. It is, for example, not known whether $TC^0\neq NC^1$.
It is in fact not even known whether $TC^0$ is equal to $P$, the class of problems decidable in polynomial time~\citep{arora_barak_2009}. 
As a consequence, certain restrictions have been placed on the generation of Boolean circuits in order to allow for a more realistic and fine-grained study of their properties. The central restriction is that of \emph{uniform generation}, or \emph{uniformity}.

\begin{defn}[Uniform Boolean Circuits]
   A family of Boolean circuits $\{C_n\}_{n \in \mathbb{N}}$ is \emph{uniform} if there exists a deterministic Turing machine that, given $1^n$, outputs a description of $C_n$ in time bounded by some resource (e.g., logarithmic space or logarithmic time). Of particular interest to us is \emph{log-space uniformity}, in which the Turing machine operates using $O(\log n)$ space for input length \(n\).
\end{defn}

\begin{figure}[ht]
    \centering
    \includegraphics[width=0.82\linewidth]{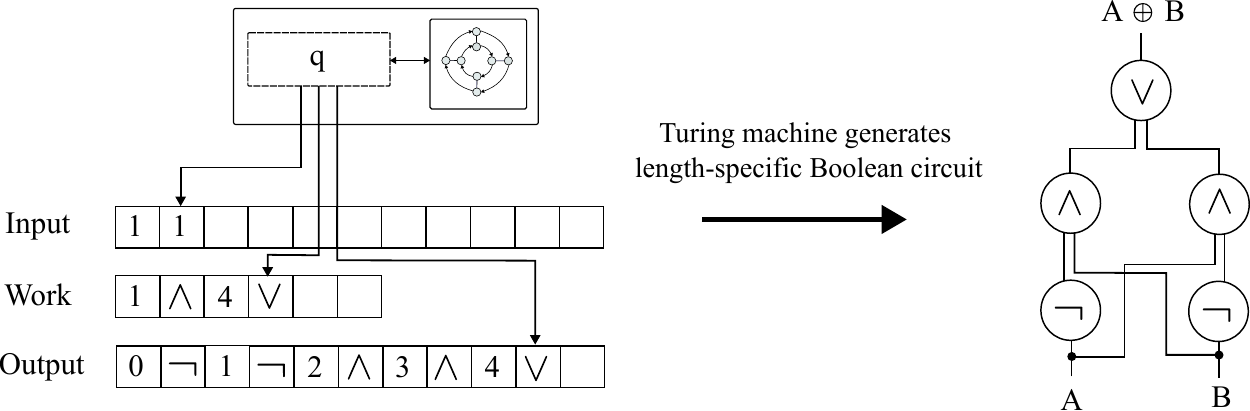}
    \caption{A visualization of a uniformly generated circuit family, with the specific example being the XOR circuit for inputs of length 2. On the left we see a Turing machine whose input is $11$, signifying that the circuit should be generated for inputs of length 2. Either the size of the work tape or the number of computational steps is upper bounded as a function of the input length. The output tape contains the description of the circuit visualized on the right. This is in fact a \emph{log-space uniform} circuit family, symbolized by the short work tape. The XOR circuit is highly related to the parity automaton, as it outputs $1$ if and only if the number of ones in the input is odd.}
    \label{fig:uniform_XOR}
\end{figure}

\subsection{SSM Expressivity Results}\label{app:ssm-expressivity}

We finally arrive at the results that this section was building up to, bounds on the expressivity of SSMs within the framework of algebraic groups and Boolean circuits. 
We do not prove any new theorems here nor restate the proofs, opting instead to refer to work which proves the claims. 
We will use the notion \emph{group world problem}, defined as follows~\cite{merrill_illusion_2024}:
\begin{defn}[Group Word Problem]
    Let $(M,\cdot)$ be a finite group. The word problem of $M$ is defined as evaluating the product of arbitrary sequences of elements of $M$. That is, given a sequence $ m_0m_1\dotsm_k$, solving the word problem involves returning $ m\in M $ such that $m_0\cdot m_1 \cdots m_k=m$.
\end{defn}

A classic result then states:
\begin{theorem}[\cite{barrington_1989}, Theorem 5]\label{the:NC1_nonsolv}
    The word problem for any nonsolvable group is complete for $NC^1$ under $AC^0$ reductions.
\end{theorem}

The separation of word problems over solvable and non-solvable groups rests on the widely-held conjecture that $TC^0\neq NC^1$~\cite{arora_barak_2009}, combined with the result that word problems of solvable groups belong to $TC^0$:

\begin{theorem}[\cite{barrington_nc1_1992}, Theorem 8]
    The word problem for any solvable group is in $TC^0$.
\end{theorem}

A recent work analyses the expressivity of diagonal SSMs by demonstrating that the individual computations in the SSM can be emulated by $\text{logspace-uniform } TC^0$ circuits, under the assumption that the numerical representation capacity is logarithmic in the sequence length.

\begin{theorem}[\cite{merrill_illusion_2024}, Theorem 4.4]\label{the:merill}
    Every fixed depth log-precision diagonal SSM can be emulated by a logspace-uniform $TC^0$ circuit.
\end{theorem}

This is the main upper bound we refer to in the paper, and is what we mean when we say that \emph{non-solvable automata cannot be emulated by diagonal SSMs}, which is implied by combining the strict inclusion $\text{logspace-uniform } TC^0\subset TC^0$ with the conjectured $TC^0 \cap NC^1\neq \emptyset$, and utilizing Theorem~\ref{the:merill}. This bound rests on an unproven conjecture, but it is supported by significant experimental evidence~\cite{merrill_illusion_2024,cirone_theoretical_2024,grazzi2024unlocking,sarrof2024expressivecapacitystatespace,terzic2025sdssm}\footnote{Using tools from \emph{rough path theory},~\citep{cirone_theoretical_2024} prove the restrictiveness of diagonal time-variant SSMs as compared to dense ones without relying on unproven circuit complexity conjectures.}. 

\section{Architectural Details}


\subsection{Full Architecture}

This work mainly describes a new method for generating the SSM transition matrices. The method is itself embedded into a larger architecture, which mostly follows standard practice in neural network design as pioneered by the Transformer~\citep{vaswani_attention_2017}. That is, we embed the $PD$ parametrization into neural network layers which consist of nonlinear feed-forward networks, vector normalization, and residual connections, calling the entire model~\name{}. The connection pattern of \name{} closely follows that of pre- or post-norm Transformers, and can be seen in~\citep{pre_post_norm}. We show two sketches of the full PD-SSM model in Figure~\ref{fig:full_model_sketch}.

\paragraph{Readout} The readout from the state, $\psi(x_t)$, is different from standard practice, which typically implements it as $\psi(x_t)=Re\{x_t\}$~\citep{orvieto_resurrecting_2023,gu_efficiently_2022,gu_mamba_2023, smith_simplified_2023}. Instead of taking the real part of $x_t$, we apply a transformation on the concatenation of the real and imaginary part of $x_t$, denoted as $Re\{x_t\}||Im\{x_t\}$. Depending on the benchmark, the readout is either a GELU-activated nonlinear network:
\begin{align*}
    \psi(x_t)=W^{Read,O}\sigma_{Gelu}(W^{Read,I}(Re\{x_t\}||Im\{x_t\})+b^{Read,I})+b^{Read,O}
\end{align*}
Or, if we refer to it as a linear readout, we concretely mean the transformation:
\begin{align*}
    \psi(x_t)=W^{Read}(Re\{x_t\}||Im\{x_t\})+b^{Read}
\end{align*}
For the long-range arena tasks, we employ the nonlinear readout and a deeper architecture. For the state tracking tasks, we apply the linear readout and a single-layer architecture.
See Appendix~\ref{app:results}.

\begin{figure}[ht]
    \centering
    \begin{subfigure}[b]{0.45\textwidth}
        \centering
        \includegraphics[width=\textwidth]{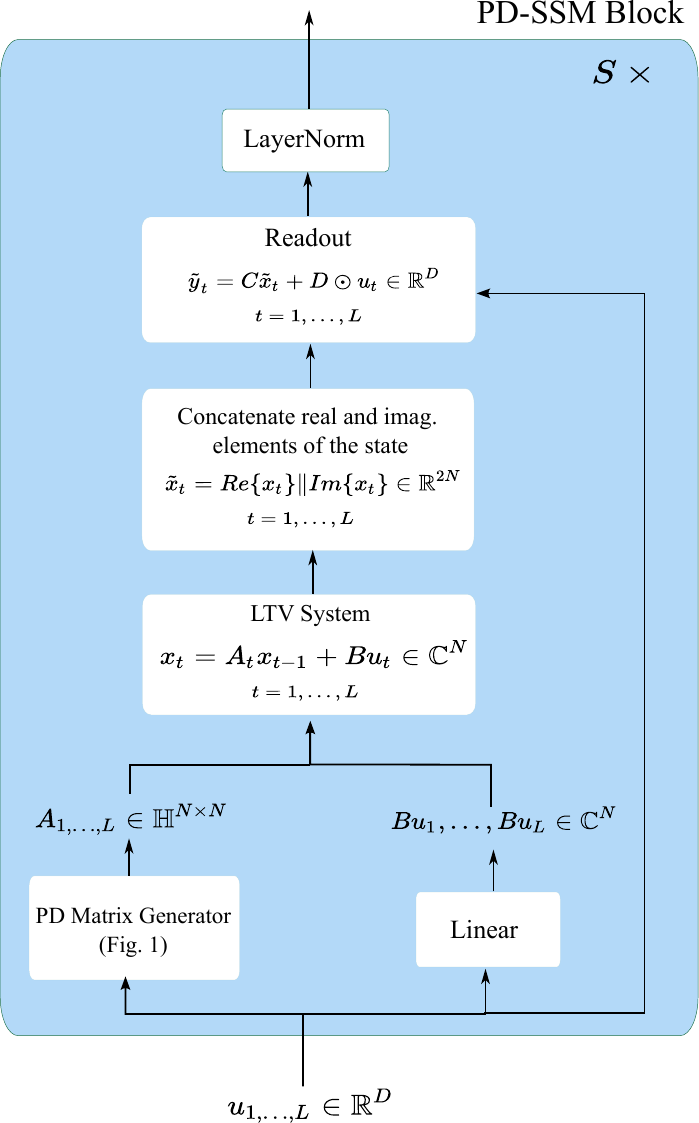}
        \caption{Full PD-SSM model with a linear readout, as used in the state-tracking experiments.}
        \label{fig:lin_pdssm}
    \end{subfigure}
    \hfill
    \begin{subfigure}[b]{0.45\textwidth}
        \centering
        \includegraphics[width=\textwidth]{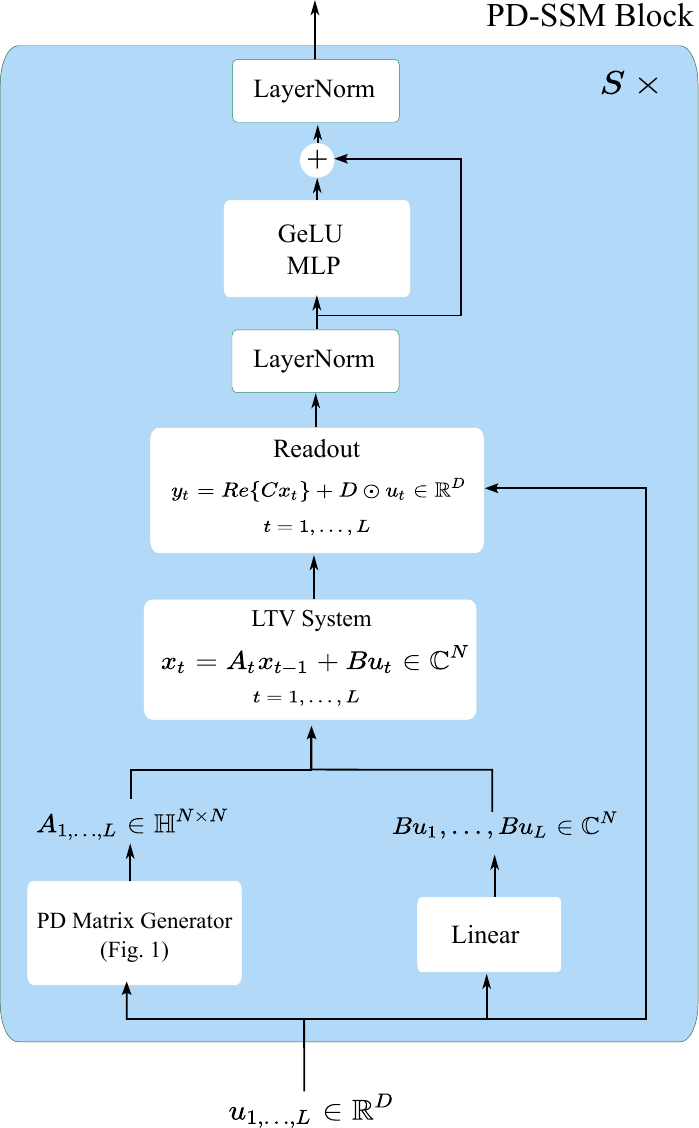}
        \caption{Full PD-SSM model with a nonlinear readout, as used in the LRA and time-series experiments.}
        \label{fig:nonlin_pdssm}
    \end{subfigure}
    \caption{Full PD-SSM models we used in our state-tracking and LRA experiments. In the time-series classification experiments, we used a model equivalent to the one shown in subfigure~\ref{fig:lin_pdssm}, with the exception that $\tilde{x}_t=Re\{x_t\}\in\mathbb{R}^N$.}
    \label{fig:full_model_sketch}
\end{figure}

\subsection{Efficiently Computing the Parallel Scan Gradients}


The total derivative of a loss $L$ with respect to $x_k$, written as $\frac{d L}{d x_k}$, satisfies the equation below:
\begin{align*}
    \frac{d L}{d x_k}=\frac{\partial L}{\partial x_k}+\frac{d L}{\partial x_{k+1}}\frac{\partial x_{k+1}}{\partial x_k}
\end{align*}

In our notation, \(\frac{d L}{d x_k}\) explains the total effect of $x_k$ on $L$ via all of the intermediate computations, whereas \(\frac{\partial L}{\partial x_k}\) denotes the direct effect of $x_k$ on $L$, ignoring the influence of $x_k$ on the rest of the sequence. 
The expression above is equivalent to:
\begin{align*}
    \frac{d L}{d x_k}=\frac{\partial L}{\partial x_k}+\frac{d L}{d x_{k+1}}A_{k+1}
\end{align*}

We can see that computing the derivative of $L$ with respect to the states $x_k$ is a \emph{linear recurrence with sparse transition matrices $A_{k+1}$}.
%
%
Computing the gradient thus reduces to a reverse parallel scan with matrix multiplication complexity of $O(N)$.
Within the framework of the parallel scan introduced in Section~2.4, gradient computation can then be parallelized using the following associative operation:


\begin{align*}
    (A_i,\frac{d L}{d x_{i-1}})\bullet (A_{i+1},\frac{d L}{d x_{i}})=(A_iA_{i+1},\frac{d L}{d x_{i}}A_i+\frac{d L}{d x_{i-1}})
\end{align*}

\section{Proofs}\label{app:proofs}

We first restate and prove the two properties of the space of column one-hot matrices, $\mathbb H^{N \times N}$.

\begin{manualproper}{1}[Algebraic Structure of One-hot Column Matrices]
    $\mathbb H^{N \times N}$ is a monoid under matrix multiplication, i.e., it is closed under associative matrix multiplication and contains the identity.
\end{manualproper}
\begin{proof}
    It is easy to verify that $\mathbb H^{N \times N}$ contains the identity, as it is a column one-hot matrix. It is also trivial to verify that the product of two column one-hot matrices remains column one-hot. Given $A,B \in \mathbb H^{N \times N}$, the columns of $B$ \emph{select and scale} the columns of $A$. More precisely, the entry in row $i$, column $j$ of the product matrix $C = AB$ is given by:
        \[
        C_{ij} = \sum_{k=1}^N A_{ik} B_{kj}
        \]
        
        Since for each fixed $j$, the matrix $B$ has exactly one non-zero entry in column $j$, say at row $k = r$, all other terms in the sum are zero. That is,
        \[
        B_{kj} = 0 \quad \text{for all } k \neq r, \quad \text{and } B_{rj} \neq 0
        \]
        
        So the sum reduces to a single term:
        \[
        C_{ij} = A_{i r} \cdot B_{rj}
        \]
        
        Now, since matrix $A$ also has exactly one non-zero entry in each column, the $r$-th column of $A$ has only one non-zero value, say at row $i = s$. Therefore,
        \[
        A_{i r} = 0 \quad \text{for all } i \neq s, \quad \text{and } A_{s r} \neq 0
        \]
        
        This implies that for each column $j$ of the product $C$, there is exactly one row $i$ such that $C_{ij} \neq 0$, and all other entries in that column are zero.
\end{proof}

\begin{manualproper}{2}[Computational Efficiency of Matrix Multiplication in $\mathbb H^{N \times N}$]
    Let $A,B\in\mathbb{H}^{N\times N}$. Then $C=AB \in \mathbb{H}^{N\times N}$ can be computed in $\Theta(N)$ arithmetic operations.
\end{manualproper}
\begin{proof}
    As per the previous proof, in order to compute $C=AB$ for $A,B \in \mathbb H^{N \times N}$, we only need to compute one element in each column of $C$. Figure 3 in the main text outlines the details of the computation.
\end{proof}

The following proof is an application of a chain of inequalities that shows that under certain assumptions, which are fulfilled by the PD-parametrization, the SSM is BIBO-stable.

\begin{manualprop}{1}[System Stability under PD-Parametrization]
    Let $\varepsilon \in (0,1]$ and consider the state transition $x_k = A_k x_{k-1}+b_k$ with $A_k \in \mathbb H^{N \times N} : \|A_k\|_\infty \leq 1- \varepsilon$. Let further $\|x_{0}\|_2 \leq B$ and $\|b_k\|_2 \leq B$ for $B<\infty$. Then it holds that
    \begin{equation}
        \|x_k\|_2 \leq \sqrt{N}B / \varepsilon \quad \forall k.
    \end{equation}
\end{manualprop}

\begin{proof}

    Expanding the recursion of the state transition results in 
    \begin{equation*}
        x_k = (\prod_{j=k}^1A_j) x_0 + \sum_{t=1}^{k} (\prod_{j=k}^{t+1} A_{j} )b_t,
    \end{equation*}
    where we define $\prod_{j=k}^{i}A_j = \mathbb{I}$ whenever $i>k$, with $\mathbb{I}$ denoting the identity matrix.
    
    We bound the norm of the state as follows:
    \begin{align}
        \|x_k\|_2 &= \| (\prod_{j=k}^1A_j) x_0 + \sum_{t=1}^{k} (\prod_{j=k}^{t+1} A_{j} )b_t \|_2 \\
        &\leq \| (\prod_{j=k}^1A_j) x_0 \|_2 + \| \sum_{t=1}^{k} (\prod_{j=k}^{t+1} A_{j} )b_t \|_2 \\
        &\leq \| \prod_{j=k}^1A_j  \| \|x_0 \|_2 + \sum_{t=1}^k \|  \prod_{j=k}^{t+1}A_j\| \|b_t \|_2  \\
        &\leq B (\| \prod_{j=k}^1A_j  \|  + \sum_{t=1}^k \|  \prod_{j=k}^{t+1}A_j\| ) \label{eq:proof_system_stability_with_pd_parametrization_triangle_inequality}
    \end{align}
    where the first inequality is the result of the triangle inequality, the second inequality results from the definition of the matrix spectral norm in both terms (specifically, denoting the spectral norm of a square matrix $A$ as $\| A\|$, it holds that for a vector $b$, $\|Ab\|_2 \leq \| A \| \| b \|_2$) and the triangle inequality in the second term, and the final inequality is due to the assumptions on the norms of $x_0$ and $b_t$.

    To bound the spectral norms of the two products of the form $\textstyle\prod_{j=k}^i A_j$, we will use the fact that for general complex-valued matrices $A$, it holds that $\| A \| \leq \|A \|_F$. We will additionally leverage the algebraic structure of PD matrices, i.e., the fact that $\textstyle\prod_{j=k}^iA_j, k\geq i$ has a column one-hot structure, which enables us to use the trivial bound $\| \textstyle\prod_{j=k}^iA_j \|_F \leq \sqrt{N} \| \textstyle\prod_{j=k}^iA_j\|_\infty$. We will combine these results to derive the following upper bound:
    \begin{equation*}
        \| \prod_{j=k}^i A_j \| \leq \| \prod_{j=k}^i A_j \|_F \leq \sqrt{N}\| \prod_{j=k}^i A_j \|_\infty \leq \sqrt{N}(1-\epsilon)^{k-i+1}
    \end{equation*}

    We will prove this by induction, assuming that the follwing formula holds for a fixed $i$ and arbitrary $k\geq i$:
    \begin{equation}
        \| \prod_{j=k}^i A_j \|_\infty \leq (1-\epsilon)^{k-i+1} \label{eq:spectral_norm_product_inequality}
    \end{equation}

    The base case $k=i$ holds as a direct consequence of the assumption's proposition:
    \begin{equation*}
        \| \prod_{j=i}^i A_j \|_\infty = \| A_i \|_\infty \leq  (1-\epsilon) 
    \end{equation*}

    Suppose now that Equation~\ref{eq:spectral_norm_product_inequality} holds for some $k>i$. Then, for $k+1$,
    \begin{align*}
        \| \prod_{j=k+1}^i A_j \|_\infty &= \| A_{k+1} \prod_{j=k}^i A_j  \|_\infty = \|P_{k+1} D_{k+1} \prod_{j=k}^i A_j \|_\infty
    \end{align*}
    where we decomposed $A_{k+1}$ into the column one-hot matrix $P_{k+1}$ and the complex diagonal matrix $D_{k+1}$.
    By the proposition's assumption, $\|D_{k+1} \|_\infty \leq 1-\epsilon$. By the induction's assumption, $ \| \textstyle\prod_{j=k}^i A_j\|_\infty \leq (1-\epsilon)^{k-i+1}$. Since $D_{k+1}$ is diagonal, $\| D_{k+1} \textstyle\prod_{j=k}^i A_j \|_\infty \leq (1-\epsilon)^{k-j+2} $. Finally, $P_{k+1}$ only rearranges the entries of $D_{k+1} \textstyle\prod_{j=k}^i A_j$, so that, as desired,
    \begin{align*}
         \| \prod_{j=k+1}^i A_j \|_\infty = \| A_{k+1} \prod_{j=k}^i A_j  \|_\infty = \|P_{k+1} D_{k+1} \prod_{j=k}^i A_j \|_\infty  \leq (1-\epsilon)^{(k+1)-j +1}
    \end{align*}

    We obtain the final result by plugging these bounds into Equation~\eqref{eq:proof_system_stability_with_pd_parametrization_triangle_inequality}, i.e., we have, with $k\geq 1$:
    \begin{align*}
        \|x_k\|_2 &\leq  B (\| \prod_{j=k}^1A_j  \|  + \sum_{t=1}^k \|  \prod_{j=k}^{t+1}A_j\| )  \\
        &\leq B(\sqrt{N}\|\prod_{j=k}^1A_j \|_\infty +\sum_{t=1}^k \sqrt{N} \|\prod_{j=k}^{t+1}A_j \|_\infty) \\
        &\leq B\sqrt{N} \left[ (1-\epsilon)^k  + \sum_{t=1}^k (1-\epsilon)^{k-t} \right] \\
        &= B\sqrt{N} \left[ (1-\epsilon)^k(1+ \sum_{t=1}^k(\frac{1}{1-\epsilon})^{t}) \right] \\
        &= B\sqrt{N} \left[ (1-\epsilon)^k(1+\frac{(\frac{1}{1-\epsilon})^k - 1}{\epsilon} \right] = B\sqrt{N} \left[ (1-\epsilon)^k(1-\frac{1}{\epsilon}) + \frac{1}{\epsilon}\right] \leq \frac{B\sqrt{N}}{\epsilon}
    \end{align*}

\end{proof}

The following lemma will be of use for proving the optimality of $PD$ parametrization.

\begin{lem}\label{prop:reduction_complex_SSM_to_real_SSM}
    Any single-layer complex SSM with state size $d$ can be reduced to a single-layer real SSM with state size $2d$.
\end{lem}
\begin{proof}
    Let $x_t, b_t \in \mathbb C^d$ and $A_t \in \mathbb C^{d \times d}$ with $x_{t+1} = A_t x_t + b_t$. Write $x^r_t = \mathcal R (x_t)$, $x^i_t = \mathcal I(x_t)$, $b^r_t = \mathcal R (b_t)$, $b^i_t = \mathcal I (b_t)$, $A^r_t = \mathcal R (A_t)$, and $A^i_t = \mathcal I (A_t)$. Then
    \begin{align*}
        x^r_{t+1} + i x^i_{t+1} & = x_{t+1} = A_t x_t + b_t = (A^r_t + i A^i_t) (x^r_{t} + i x^i_{t}) + b^r_t + i b^i_t\\ & = A_t^r x_t^r - A_t^i x_t^i + b_t^r + i(A_t^i x_t^r + A_t^r x_t^i + b_t^i).
    \end{align*}
    In block matrix form, the complex SSM can be represented as a real SSM through
    \begin{equation*}
        \begin{pmatrix}
            x_{t+1}^r\\
            x_{t+1}^i
        \end{pmatrix}
        = \begin{bmatrix}
            A_t^r & - A_t^i\\
            A_t^i & + A_t^r
        \end{bmatrix} \begin{pmatrix}
            x_{t}^r\\
            x_{t}^i
        \end{pmatrix}
        + 
        \begin{pmatrix}
            b_t^r\\
            b_t^i
        \end{pmatrix}.
    \end{equation*}
\end{proof}

The following proposition states that for each $N\in\mathbb{N}$, one can construct an FSA which requires a state dimensionality of at least $N-1$ to emulate using an SSM, under the assumption that each automaton state has a unique vector encoding. By combining this with Proposition 2, which states that all $N$-state automata can be emulated using a $PD$-SSM with one layer and state size $N$, the optimality of the $PD$ decomposition for emulating FSAs is established.

\begin{manualprop}{3}[Optimality of PD Parametrization]
    For any $N$ there exists a finite-state automaton with $N$ states that cannot be emulated by any single-layer SSM with state size less than $N-1$ under unique state encodings.
\end{manualprop}
\begin{proof}
    Interpreting complex state size as $2d$ and real state size as $d$ for $x \in \mathbb R^d$ and $x \in \mathbb C^d$, respectively, we can invoke Lemma~\ref{prop:reduction_complex_SSM_to_real_SSM} and stick to real SSMs in the remainder of this proof. Let $N \in \mathbb N$. For $N \leq 2$ the statement is trivial, so suppose $N \geq 3$. Consider the FSA $\mathbb A$ with states $s_1, \ldots, s_N$ and input vocabulary equal to the set of states mapping to state transitions 
    \begin{equation*}
        s_x \mapsto f_x \text{ where }f_x(s_i) = 
        \begin{cases}
            s_{i+1 \!\!\!\! \mod N} & i=x\\
            s_i & i \not = x.\\
        \end{cases}
    \end{equation*}
    Now, assume by contradiction a unique (one-to-one) encoding $s_i \cong v_i \in \mathbb R^d$ with $d \leq N-2$. Then there exist two states $s_x,s_y$, w.l.o.g. set to $s_{N-1}, s_N$, such that 
    \begin{equation*}
        \mathbb V := \mathrm{span}\{v_1, \ldots, v_N\} = \mathrm{span}\{v_1, \ldots, v_{N-1}\} \text{ and } v_{N-1} = \sum_{i=1}^{N-2} \alpha_i v_i \text{ as well as } v_{N} = \sum_{i=1}^{N-2} \beta_i v_i.
    \end{equation*}
    Now, a single-layer SSM that emulates $\mathbb A$ must satisfy
    \begin{equation*}
        A_c v_i + b_c = 
        \begin{cases}
            v_i & i \not=c\\
            v_{i+1 \!\!\!\!\!\mod N} & i = c.
        \end{cases}
    \end{equation*}
    An immediate consequence is that  $\textstyle \sum_{i=1}^{N-2} \alpha_i \not = 1 \not = \sum_{i=1}^{N-2} \beta_i$, since otherwise either
    \begin{equation*}
        v_N = \textstyle A_{N-1} v_{N-1} + b_{N-1} = \sum_{i=1}^{N-2} \alpha_i (A_{N-1}v_i + b_{N-1}) + b_{N-1} (1-\sum_{i=1}^{N-2} \alpha_i) = v_{N-1} + 0
    \end{equation*}
    or
    \begin{equation*}
        v_1 = \textstyle A_{N} v_{N} + b_{N} = \sum_{i=1}^{N-2} \beta_i (A_{N}v_i + b_{N}) + b_{N} (1-\sum_{i=1}^{N-2} \beta_i) = v_{N} + 0,
    \end{equation*}
    violating the uniqueness assumption. But then
    \begin{equation*}
        \mathbb V = \mathrm{span}\{v_1 - v_{N-1}, \ldots, v_{N-2} - v_{N-1}\},
    \end{equation*}
    since $\textstyle\sum_{i=1}^{N-2} \alpha_i (v_i - v_{N-1}) = v_{N-1}(1-\textstyle\sum_{i=1}^{N-2} \alpha_i) \not=0$. We can now leverage this spanning property by noting that $A_{N} (v_i - v_{N-1}) = (v_i - v_{N-1})\ \forall i \in \{1, \ldots, N-2\}$ and thus $A_N \vert_{\mathbb V} = I \vert_{\mathbb V}$. But if $A_N$ acts as identity on $\mathbb V \ni v_N$, then $v_1 = A_N v_1 + b_N = v_1 + b_N$ and hence $b_N = 0$. But then we reach the contradiction violating uniqueness of state encodings
    \begin{equation*}
        v_{1} = A_N v_N + b_N = A_N \vert_\mathbb V v_N + b_N = I\vert_\mathbb V v_N + b_N = v_N.
    \end{equation*}
\end{proof}

The following proposition is meant to highlight the importance of assuming a restricted readout, such as done in Proposition~\ref{prop:optimality_pd_parametrization} with unique state encodings. Indeed, by using lookup tables that are exponentially large in the maximum sequence length, even a scalar SSM can emulate any finite state automaton. In practice, however, such lookup tables are entirely impractical.

\begin{manualprop}{4}[Arbitrary Precision and Readout]
    Consider $x_{t+1} = x_t + b_t$ where $b_t = u_t \cdot k_t$ with $u_t \in \mathbb Q$ input encodings and $k_t = \sqrt{p_t} \in \mathbb R /\mathbb Q$ time encodings, where $p_t$ is the $t$-th prime. Then, any FSA can be encoded into this scalar-valued SSM under an appropriate lookup table as readout.
\end{manualprop}
\begin{proof}
    According to \cite{jaffe2007linearly}, $k_t$ is linearly independent over $\mathbb Q$, i.e., for $u_i \in \mathbb Q$ it holds that $u_1 k_1 + \cdots + u_n k_n = 0 \iff u_1 = \cdots = u_n = 0.$ Therefore, no two differing input sequences $u_1, \ldots, u_n$ and $v_1, \ldots, v_m$ will result in the same state representation $\textstyle x_t = x_0 + \sum_i u_i k_i$. Now, a simple identification of every FSA state with all the state representations $x_t$ of associated input sequences that lead to said FSA state shows universal expressivity under arbitrary readout. 
\end{proof}

\section{Experimental Setup}

\subsection{Hyperparameter Selection}

\paragraph{FSA Emulation}

On FSA emulation, Table~\ref{tab:fsa-table}, we did not perform a hyperparameter grid search. We re-used the fixed hyperparameters which were used to evaluate all of the baseline methods, as per~\cite{walker_2025_structured}. In contrast to the baseline methods which were evaluated using various state sizes including 128, we only evaluated our model using only state size $128$. We used Adam with the default parameters (0.9, 0.999).

On experiments with non-solvable groups in Table~\ref{tab:nonsolvable}, we ran a learning rate grid search in $\{1e{-4}, 5e{-4}, 1e{-3}, 5e{-3}, 1e{-2}\}$, and report the best validation accuracy over any of the 3 random seeds and hyperparameter configurations. The state size was 128. We used Adam with the default parameters (0.9, 0.999).

\paragraph{Time-Series Classification}

We performed the following grid search, following~\cite{rusch2025oscillatory}: Learning rate in $\{1e{-5}, 1e{-4}, 1e{-3}\}$, number of layers in $\{2, 4, 6\}$, state dimension in $\{16, 64, 128\}$ and embedding dimension in $\{16, 64, 128\}$. We used Adam with the default parameters (0.9, 0.999).

\paragraph{Long-Range Arena}

We performed a grid search defined by all combinations of the following values: Learning rate in $\{1e{-5}, 5e{-5}, 1e{-4}, 5e{-4}, 1e{-3}\}$, weight decay in $\{0.001, 0.01, 0.1, 0.05\}$, and dropout in $\{0.0, 0.1, 0.3\}$. On each task, we report the average accuracy over three random seeds. The embedding dimension and state size were both fixed to 128, with the exception of the Retrieval dataset in which we reduced it to 64. We used Adam with the default parameters (0.9, 0.999). The hyperparameters of the models we compare with were obtained via an extensive Bayesian hyperparameter search, as outlined in~\cite{soydan_2024_simplified}.

\paragraph{Natural Language State Tracking}

The setup of this task is equivalent to symbolic state tracking, only we train for 100,000 steps or until convergence on a validation set with length 40. We only ran a grid search over learning rates in $\{1e-4, 5e-4, 1e-3\}$. State size was fixed to $128$, and the number of dictionary matrices (K) was $6$.



\subsection{Training Times}

We report the training time for representative benchmarks using \texttt{NVIDIA A100 40GB} GPUS.

\paragraph{Natural Language State Tracking} This benchmark uses an inefficient PyTorch implementation of the PD-SSM, resulting in extended training times. The experiment can be executed significantly more efficiently using the provided JAX implementation.

On the \emph{Parity} automaton:
\begin{itemize}
    \item PD : 7.2 hours
    \item Complex diagonal: 70.3 hours
    \item Real diagonal: 68.6 hours
\end{itemize}

On the $A_5$ automaton:
\begin{itemize}
    \item PD : 82.8 hours
    \item Complex diagonal: 68.4 hours
    \item Real diagonal: 63.6 hours
\end{itemize}

\paragraph{Long-Range Arena}

\begin{itemize}
    \item CIFAR: 2.5 hours
    \item ListOps: 9.3 hours
    \item IMDB: 6.6 hours
    \item AAN: 34 hours
\end{itemize}

\paragraph{UEA Time-Series} All results can be obtained within several hours on the specified GPU.

\section{Additional Results}\label{app:results}

\subsection{Runtime Measurements}

\paragraph{Measurement Details} The runtime measurements were obtained using optimized and compiled JAX code\footnote{\hyperlink{https://docs.jax.dev/en/latest/quickstart.html}{https://docs.jax.dev/en/latest/quickstart.html}}. 
We did not implement custom system-aware methods, which can certainly make a difference in practice~\citep{gu_mamba_2023, dao2023flashattention2}.
We assume that our runtime measurement setup is reflective of most researchers' typical use cases, and should provide an equal testbed for all the methods.
All of the runtime measurements were obtained on an \texttt{NVIDIA A100 80GB} GPU. The code is implemented in \texttt{JAX}, version 0.4.24. The parallel scan relies on the \texttt{jax.lax.associative\_scan} primitive. The computation is parallelized across the batch dimension and sequence length using \texttt{jax.vmap}, and is finally just-in-time compiled.
At the time of writing, the parallel scan primitive is a work-in-progress in PyTorch. In our experience, the JAX version was significantly more efficient.

\paragraph{State Representation Size Matched Comparison} We extend the comparison of our PD model with two representative methods, a time-variant SSM with unstructured but column-normalized transition matrices (SD-SSM)~\cite{terzic2025sdssm} \footnote{Column normalization is a crucial step to ensure system stability when using unstructured transition matrices, and has been implemented by dividing the matrix elements by the $L^p$ (quasi-)norm of the corresponding column~\citep{fan_advancing_2024,terzic2025sdssm}. As a consequence of the Gershgorin disc theorem and basic inequalities, a matrix with $L^p$-normalized columns is guaranteed to have spectral radius less than 1 whenever $p<1$.}, and a time-variant SSM with diagonal transition matrices generated using our MLP generation $A(u_t)=D(u_t)$.
Below, we extend the results from Figure~\ref{fig:runtime_dimension} by considering sequence lengths $256$ and $512$. Concretely, the dimension of SD-SSM is again set to twice that of the other models to account for the discrepancy between the real and complex-valued state representations. The results are shown in Figure~\ref{fig:state_half_runtime}.

\begin{figure}[ht]
    \centering
    \begin{subfigure}[b]{0.48\textwidth}
        \centering
        \includegraphics[width=\textwidth]{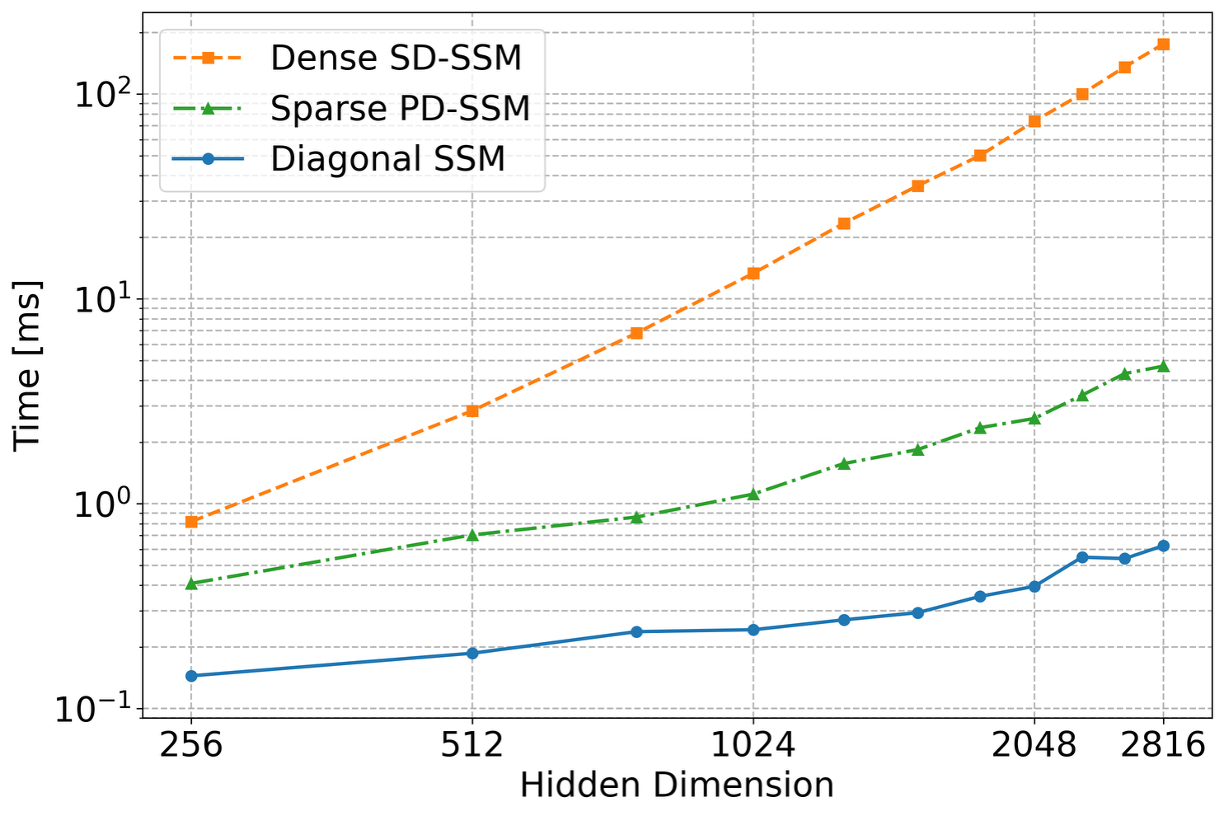}
        \caption{Sequence Length 256}
        \label{fig:state_half_256}
    \end{subfigure}
    \hfill
    \begin{subfigure}[b]{0.48\textwidth}
        \centering
        \includegraphics[width=\textwidth]{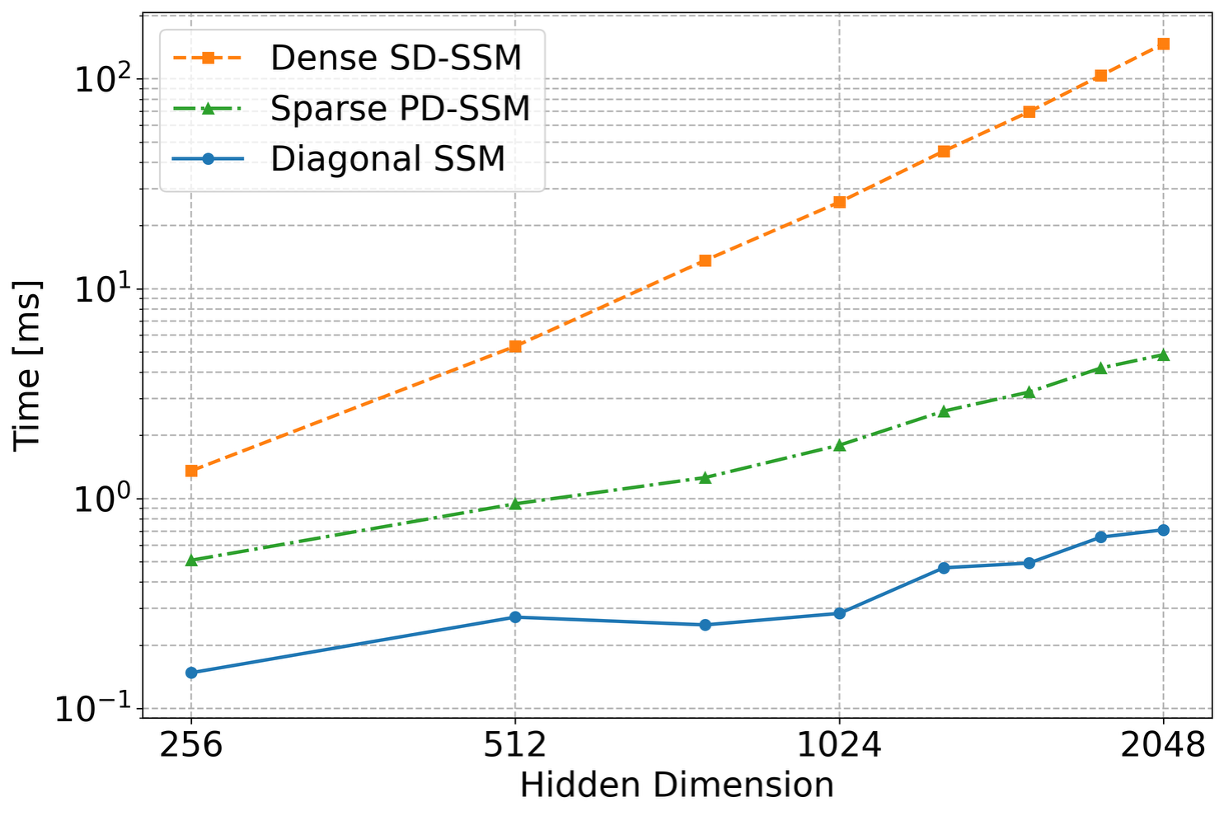}
        \caption{Sequence Length 512}
        \label{fig:state_half_512}
    \end{subfigure}
    \caption{Runtimes with state dimension of SD-SSM set to twice that of other models.}
    \label{fig:state_half_runtime}
\end{figure}

From Figure~\ref{fig:runtime_dimension}, at sequence length 64 and embedding dimension 5632, the SD-SSM exhibits an around $71\times$ slowdown compared to PD-SSM. The PD-SSM in turn is around $7\times$ slower than the diagonal model.

At hidden dimension $D=2048$, the relative speed-up of PD-SSM over SD-SSM is $18.8 \times$ for $L=64$,  $28.3 \times$ for $L=256$, and $30.3 \times$ for $L=512$, an increasing trend in $L$.

\paragraph{State Size Matched Comparison} Previously, the unstructured SSM had double the state dimensionality of the other two models to account for the discrepancy between the real and complex-valued state representations. We now set the state dimensionalities of all models to be equal and again evaluate the scaling as a function of embedding size under sequence lengths $64$ and $256$. The results are shown in Figure~\ref{fig:state_match_runtime}.

\begin{figure}[t]
    \centering
    \begin{subfigure}[b]{0.48\textwidth}
        \centering
        \includegraphics[width=\textwidth]{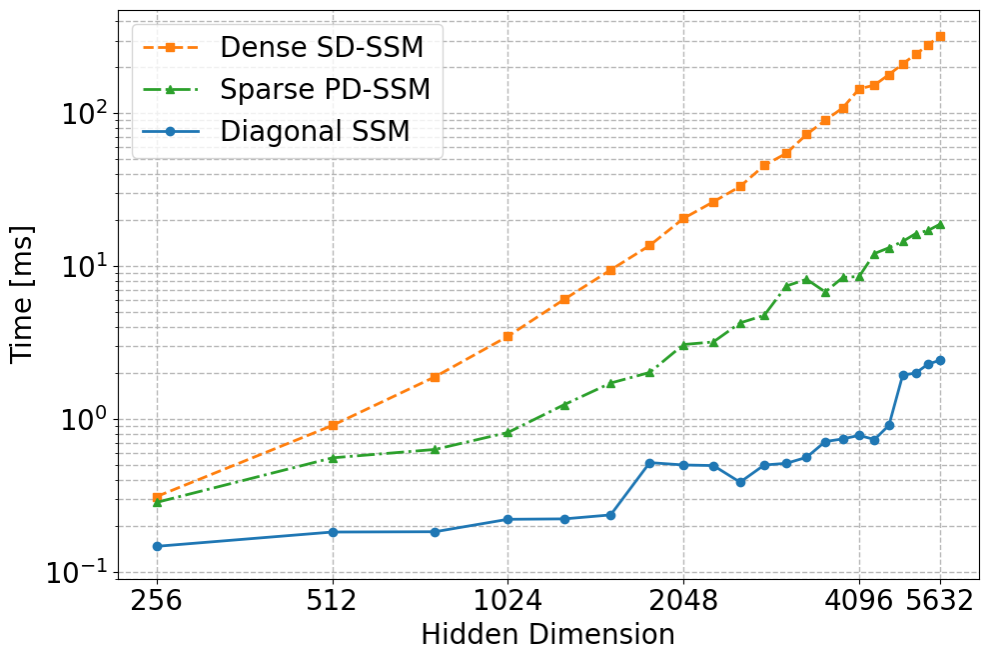}
        \caption{Sequence Length 64}
        \label{fig:state_match_64}
    \end{subfigure}
    \hfill
    \begin{subfigure}[b]{0.48\textwidth}
        \centering
        \includegraphics[width=\textwidth]{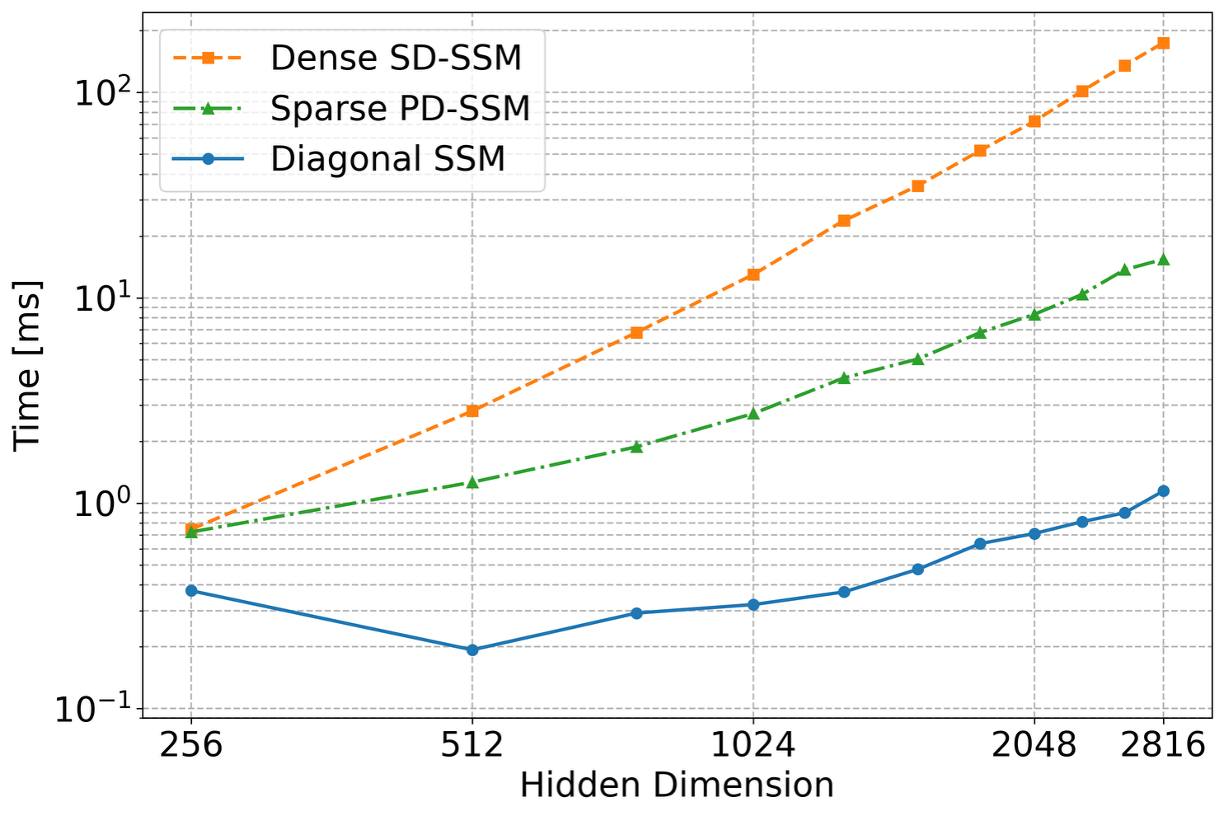}
        \caption{Sequence Length 256}
        \label{fig:state_match_256}
    \end{subfigure}
    \caption{Runtime measurements with equal state size.}
    \label{fig:state_match_runtime}
\end{figure}

At sequence length 64 and embedding dimension 5632, the SD-SSM exhibits a $16.6\times$ slowdown compared to PD-SSM. The PD-SSM is in turn $7.6\times$ slower than the diagonal model, a similar ratio as in the state matched setup above but now with double the state size.

At hidden dimension $D=2048$, the relative speed-up of PD-SSM over SD-SSM is $6.8 \times$ for $L=64$,  $8.8 \times$ for $L=256$, and $9.7 \times$ for $L=512$ (not shown). Again, an increasing trend with $L$.

\paragraph{Parameter Matched Comparison} Given state size $N$, embedding dimension $D$ and $K$ transition matrices in the dictionary where applicable, the models have the following total parameter cost: $N(2D +KN) + KD$ for SD-SSM, $N(6D + 2N + KN + 4) + KD$ for PD-SSM, and $N(6D + 2N + 4)$ for the diagonal SSM.
Let us denote the state size of SD-SSM as $N_r$ and that of PD-SSM as $N_c$. Suppose further that $D=N_r$ and $K=6$, as used in our measurements. To equalize the parameter cost of PD-SSM with SD-SSM, $N_c$ must be set as:
\[
N_c=\left\lfloor \frac{-(6N_r+4)+\sqrt{(6N_r+4)^2+192N_r^2}}{12} \right\rfloor
\]
With this, given for example $N_r=D=2048$, the corresponding $N_c$ is $1552$.
The corresponding relationship for the state dimensionality of the diagonal model $N_d$ given $D=N_r$ and $K=6$ is:
\[
N_d= \left\lfloor \frac{-(6N_r+4)+\sqrt{(6N_r+4)^2+64N_r^2+48N_r}}{4} \right\rfloor
\]
which for $N_r=2048$ exactly sets $N_d=2048$.
This adapted scaling results in the Figure~\ref{fig:param_match_runtime}.

\begin{figure}[t]
    \centering
    \begin{subfigure}[b]{0.48\textwidth}
        \centering
        \includegraphics[width=\textwidth]{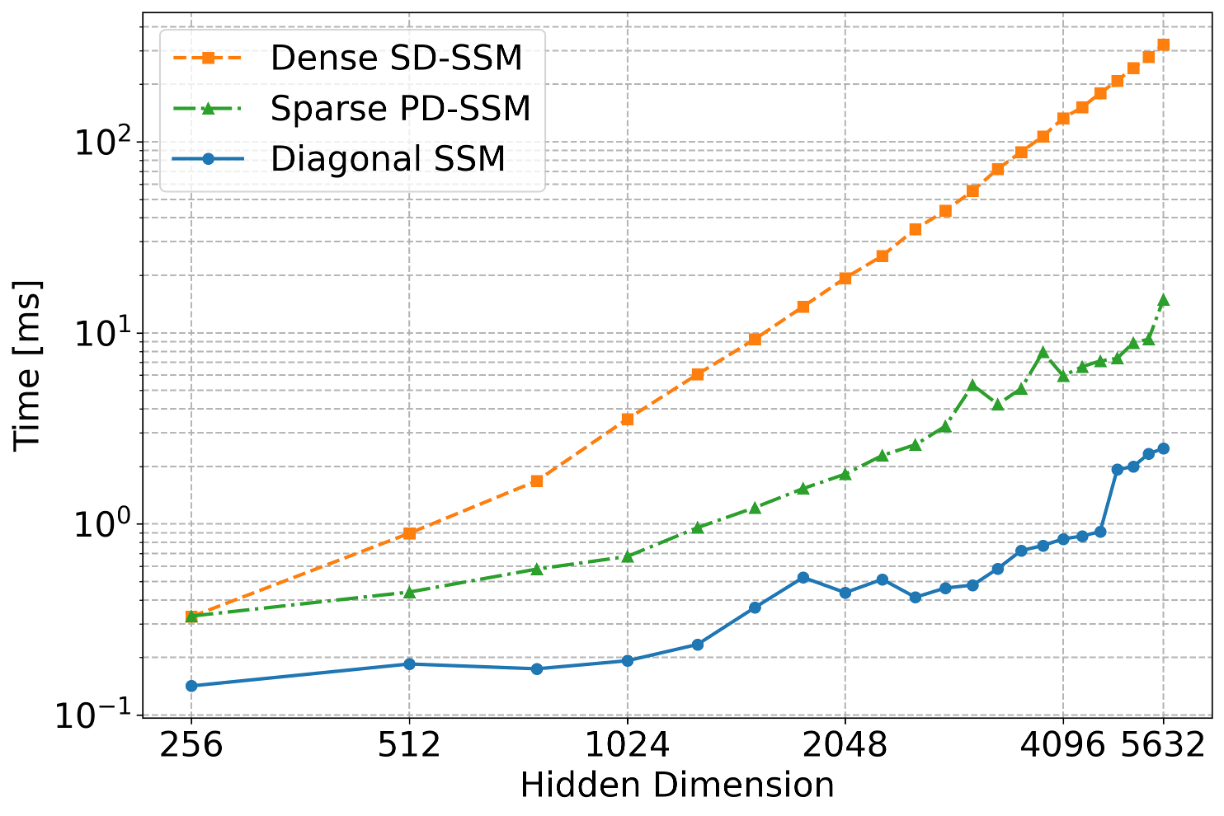}
        \caption{Sequence Length 64}
        \label{fig:param_match_64}
    \end{subfigure}
    \hfill
    \begin{subfigure}[b]{0.48\textwidth}
        \centering
        \includegraphics[width=\textwidth]{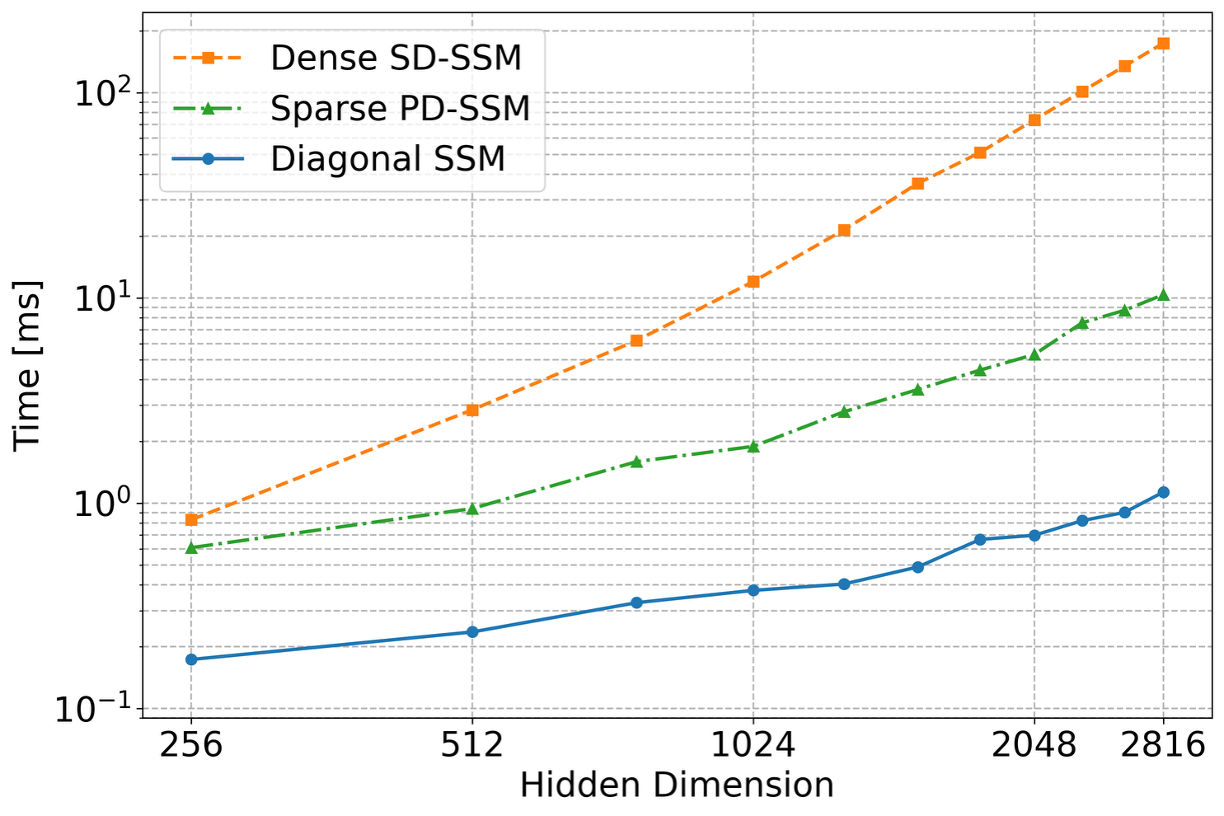}
        \caption{Sequence Length 256}
        \label{fig:param_match_256}
    \end{subfigure}
    \caption{Runtime measurements in a parameter-matched setting.}
    \label{fig:param_match_runtime}
\end{figure}

At sequence length 64 and embedding dimension 5632, the SD-SSM exhibits a $21.7\times$ slowdown compared to PD-SSM. The PD-SSM is now around $6.0\times$ slower than the diagonal model.

At hidden dimension $D=2048$, the relative speed-up of PD-SSM over SD-SSM is $10.6 \times$ for $L=64$,  $13.9 \times$ for $L=256$, and $15.5 \times$ for $L=512$ (not shown). Once again, an increasing trend with $L$.

\paragraph{Effect of $P$ Generation Under Equal State Size} In this measurement, we compare the runtime of the PD-SSM and the diagonal SSM with a version of PD-SSM that receives pre-computed random sparse $P$ matrices as input.
In this scenario, we set the state size of all models to be equal. 
The results are reported in Figure~\ref{fig:p_overhead}.
\begin{figure}[t]
    \centering
    \begin{subfigure}[b]{0.48\textwidth}
        \centering
        \includegraphics[width=\textwidth]{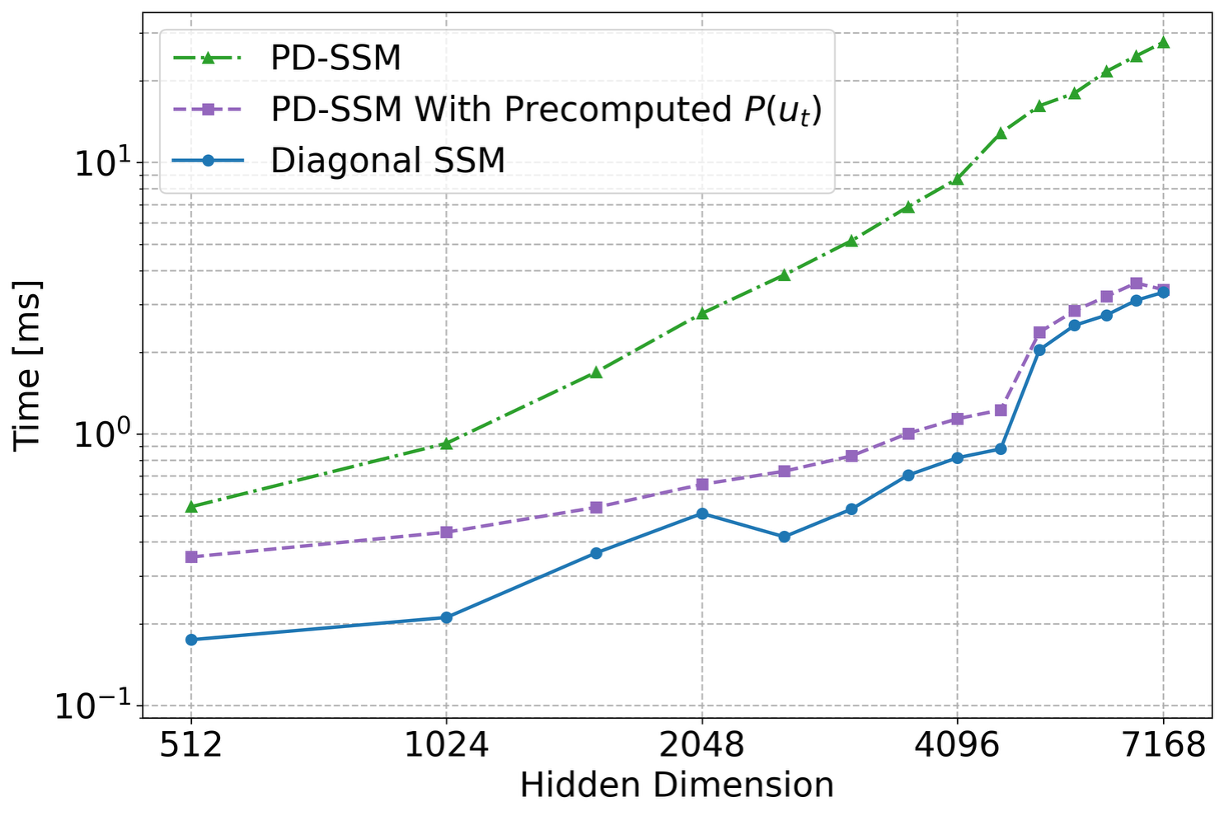}
        \caption{Runtime comparison with length $L=64$}
        \label{fig:p_log_64}
    \end{subfigure}
    \hfill
    \begin{subfigure}[b]{0.48\textwidth}
        \centering
        \includegraphics[width=\textwidth]{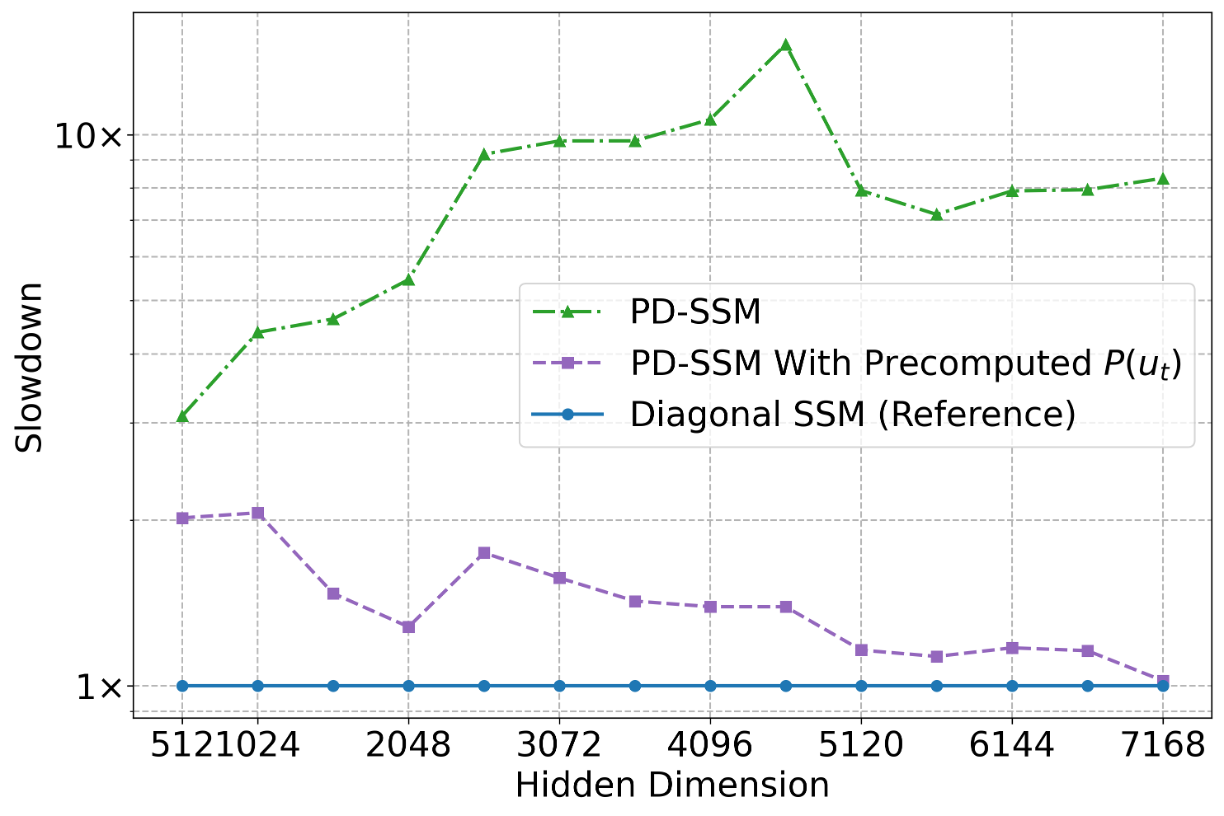}
        \caption{Slowdown w.r.t. the diagonal baseline, $L=64$.}
        \label{fig:pd_rand_64}
    \end{subfigure}
    \caption{Measuring the overhead of generating $P$ matrices.}
    \label{fig:p_overhead}
\end{figure}
As predicted in Table~\ref{tab:ssm_fsa_tracking_expressivity}, the $PD$ parallel scan incurs a constant, even diminishing overhead compared to the diagonal scan. The overhead stems primarily from the generation of the $P(u_t)$ matrices.



\subsection{Varying Mamba Depth on \texorpdfstring{$(A_5,2)$}{(A5,2)}}

Mamba~\citep{gu_mamba_2023} uses non-negative real-valued transition matrices, which are in theory highly restrictive even for emulating solvable automata~\citep{grazzi2024unlocking}\footnote{Such models must converge to a steady state upon application of a constant input. This goes contrary to FSA emulation. For example, \emph{Parity} never converges upon repeated application of a 1.}. However, we can observe that several layers of the model can in fact learn to emulate the non-solvable automaton on in-domain lengths, as shown in Figure~\ref{fig:mamba_multilayer}. The model however does not exhibit any substantial length generalization.

\begin{figure}[ht]
    \centering
    \includegraphics[width=0.7\linewidth]{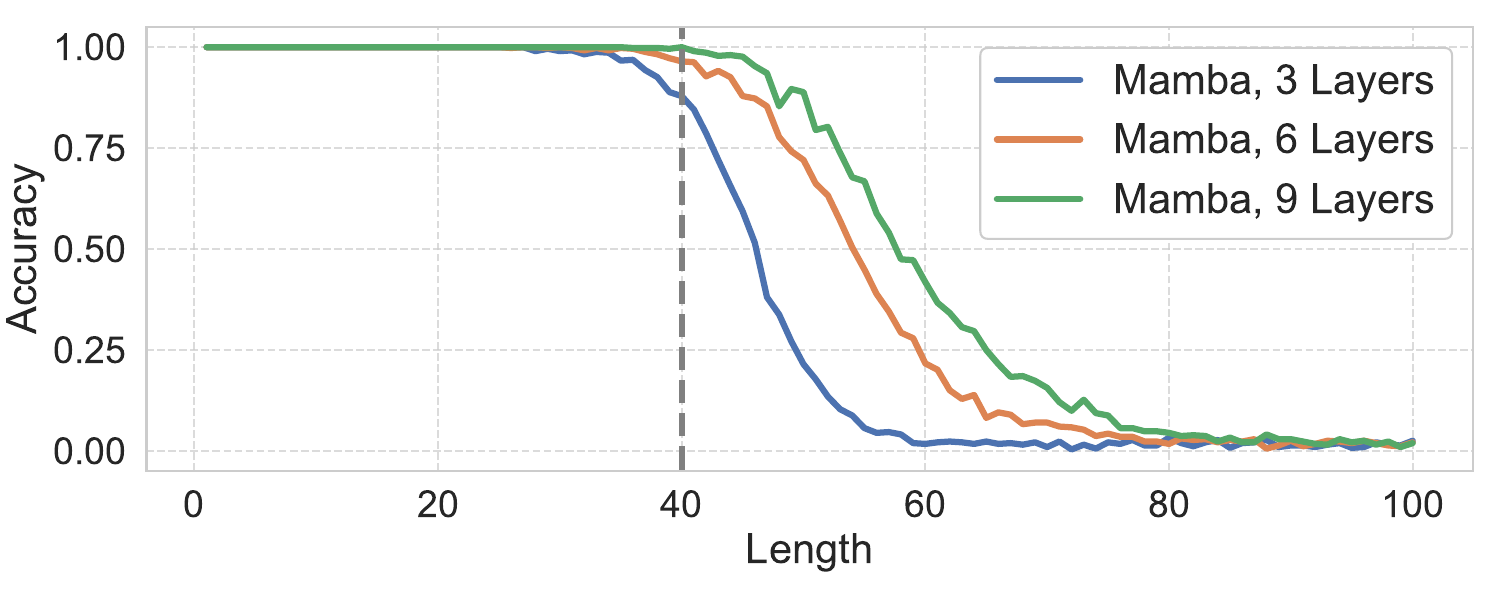}
    \caption{Mamba with varying number of layers trained to emulate the non-solvable $(A_5,2)$ automaton.}
    \label{fig:mamba_multilayer}
\end{figure}

%

\subsection{PD-SSM Ablations on \emph{Arithmetic}}

We ablate several design choices on the \emph{Arithmetic} task from~\cite{deletang_neural_2023}. 
\paragraph{Stochastic $P$ Generation and Inference With Dense $P$ Matrices:} The first ablation is centered on the generation of $P$ matrices, in particular the following \emph{deterministic} equations:
\begin{align*}
    P &= \text{hardmax}(M)\in\{0,1\}^N \text{ where } \text{hardmax(x)}_{i} := \mathds{1}_{i = \arg\max_j x_j} \\
    \frac{\partial P}{\partial M} &= \frac{\partial \text{hardmax} (M)}{\partial M} \approx \frac{\partial \text{softmax}(M)}{\partial M} 
\end{align*}
The dependence on $u_t$ as well as the indices are omitted for notational simplicity.
Typically, straight-through gradient approximators, an example of which is our approximation of the hardmax via a softmax, are used in conjunction with categorical sampling schemes~\citep{paulus_rao_2021, gumbel_softmax, bengio_stochastic_2013}. 
In particular, a widely used \emph{stochastic} scheme takes the following form, called \emph{Gumbel max}:
\begin{align*}
    P&=\text{hardmax}(M+G)\in\{0,1\}^N\\
    \frac{\partial P}{\partial M} &= \frac{\partial \text{hardmax} (M+G)}{\partial M} \approx \frac{\partial \text{softmax}(M+G)}{\partial M} \\
    G_i &\sim \mathrm{Gumbel}(0, 1), i\in[N]
\end{align*}
Given the prominence of this scheme, we test its effectiveness on learning to emulate the \emph{Arithmetic} automaton.
Simultaneously, we test the effectiveness of using dense $P$ matrices during inference. 
While the $P$ matrices must be sparse in order to ensure high efficiency during parallel training, at inference, the structure of the $P$ matrix can be relaxed\footnote{The memory footprint of dense $N\times N$ matrices is $\Theta(N^2 L)$ during training but only $\Theta(N^2)$ during inference.} .
We test the effectiveness of employing the following $P$ matrix generation during inference, $P_{:,j}=\text{softmax}(M_{:,j})\in\mathbb{R}^N$.
This is optionally combined with stochasticity as shown above.
The best obtained results with all four combinations of deterministic/stochastic $P$ generation and hardmax/softmax $P$ activation are shown in Figure~\ref{fig:gumbel_softmax_ablation}.
\begin{figure}[ht]
    \centering
    \includegraphics[width=1.0\linewidth]{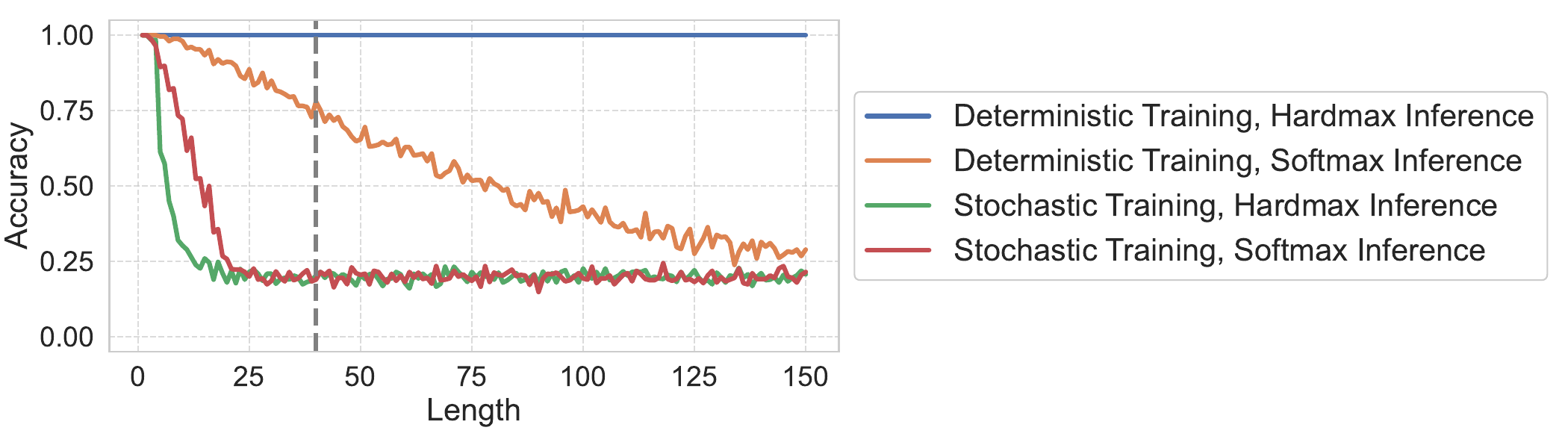}
    \caption{$P$ parametrization ablations on the \emph{Arithmetic} FSA with 1 layer of PD-SSM with a linear readout. The dashed vertical line indicates the maximum training length.}
    \label{fig:gumbel_softmax_ablation}
\end{figure}

We can see that introducing stochasticity and relaxing sparsity both have detrimental effects on accuracy. 

\paragraph{Nonlinear Readout Experiments:} The second ablation is centered on the readout function $\psi(x_t)$. As previously stated, for our state-tracking experiments, the readout function is defined as
\begin{align*}
    \psi(x_t)=W^{Read}(Re\{x_t\}||Im\{x_t\})+b^{Read}
\end{align*}
with $||$ denoting vector concatenation.
Typically, however, neural networks interleave sequence processing layers with nonlinearly activated feed-forward networks~\citep{vaswani_attention_2017}, as opposed to the affine transformation that we applied.
We test the effect of a nonlinear readout, defined as 
\begin{align*}
    \psi(x_t)=W^{Read,O}\sigma_{Gelu}(W^{Read,I}(Re\{x_t\}||Im\{x_t\})+b^{Read,I})+b^{Read,O}
\end{align*}

We repeat all of the experiments that we report in Figure~\ref{fig:gumbel_softmax_ablation}, but now with the nonlinear readout. We see in Figure~\ref{fig:mlp_gumbel_softmax_ablation} that the results are significantly worse than with a linear readout. No method performs well even on in-domain lengths. 

\begin{figure}[ht]
    \centering
    \includegraphics[width=1.0\linewidth]{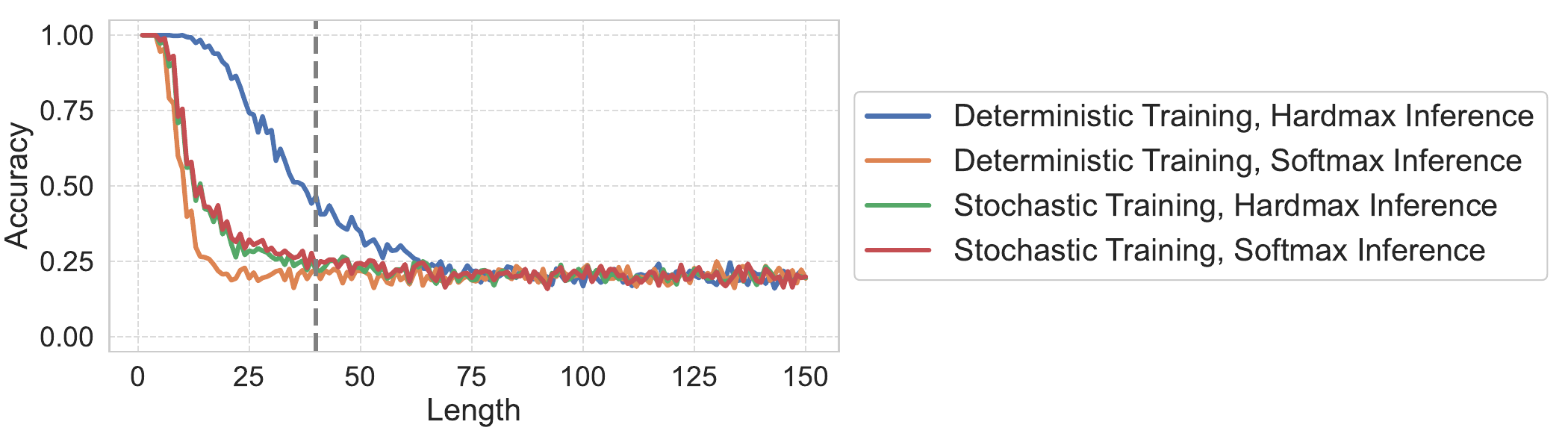}
    \caption{$P$ parametrization ablations on the \emph{Arithmetic} FSA with 1 layer of PD-SSM with a nonlinear readout. The dashed vertical line indicates the maximum training length.}
    \label{fig:mlp_gumbel_softmax_ablation}
\end{figure}

\section{Related Work}

\paragraph{General Overview of SSMs} LTI SSMs have been proposed as a more scalable alternative to Transformers that is also more performant on long sequences~\cite{gu_hippo_2020,gu_combining_2021,gu_efficiently_2022,gu_parameterization_2022,gupta_diagonal_2022,smith_simplified_2023,orvieto_resurrecting_2023}.  
A subsequent analysis showed promising language modeling performance with hybrid architectures based on gated SSMs and Attention~\cite{ma2022mega, fu_hungry_2023}.
Time-variant hybrid systems were then shown to perform significantly better on practical tasks with the Mamba SSM which utilizes diagonal transition matrices~\cite{gu_mamba_2023}.
Consequently, a line of work has demonstrated that hybrid models utilizing diagonal time-variant SSMs and attention exhibit better generalization to longer sequences, better general language modeling and reasoning performance~\citep{ren_2024_samba, de_griffin_2024, lenz_2025_jamba, wu2025transxssm, waleffe_2024_empirical}.

\paragraph{Overview of Formal Results and SSM Matrix Structures} Concurrently, formal analyses of time-varying SSMs have shown limitations of the diagonal transition matrix structure~\cite{merrill_illusion_2024,cirone_theoretical_2024}, spurring a line of work investigating the utility of block-diagonal and unstructured matrices~\cite{fan_advancing_2024,terzic2025sdssm}, as well as more expressive structured transition matrices, most notably diagonal plus low-rank~\cite{yang2024parallelizing,grazzi2024unlocking,peng_2025_rwkv,siems_2025_deltaproduct}, with the original inspiration stemming from at least as far back as fast-weight programmers~\cite{schlag2021linear, schlag2021learning, schmid_1992} whose computational structure contains products of generalized Householder matrices~\cite{yang2024parallelizing}, which enables the adoption of early algorithmic optimizations of such procedures~\cite{wy_rep_1087}. 
More recently,~\cite{movahedi2025fixedpoint} proposes fixed-point RNNs that can dynamically trade off expressivity with efficiency.

\paragraph{Details of Formal Results} As a lower bound on diagonal SSMs, \cite{sarrof2024expressivecapacitystatespace} show that a stack of complex diagonal SSM layers can emulate any solvable automaton, with the stack depth proportional to the Krohn-Rhodes (KR) complexity of the transformation semigroup~\cite{krohn_algebraic_1965,margolis_2024_decidability}.
As an upper bound on diagonal SSMs,~\cite{merrill_illusion_2024} derives a bound on fixed-depth diagonal selective SSMs with logarithmic precision representation, placing them in the $\text{L-uniform }TC^0$ circuit complexity class. 
This complexity class encompasses solvable automata. 
DPLR matrices can emulate any automaton with depth proportional to KR-complexity and width linear in depth and arbitrary readout~\cite{siems_2025_deltaproduct}.
%


\paragraph{Experimental Expressivity Analyses} Experimental analysis of sequential neural networks trained on formal language tasks goes back to at least~\citep{smith_zipser_1989, hochreiter_long_1997}. The exact experimental setup we build on is a more recent one, with the set of tasks taken from~\cite{deletang_neural_2023}, and the experimental setup borrowed from~\cite{walker_2025_structured}. All of the non-solvable group word problems were generated using modified sample generation code from~\cite{liu_transformers_2023}. 


\end{document}